\documentclass[a4paper]{amsart} 

\usepackage{amsmath,amsthm,amssymb,amsfonts,mathrsfs,color,hyperref, mathtools,crop, graphicx, enumitem, todonotes}
\usepackage[]{geometry}

\theoremstyle{plain}
\begingroup
\newtheorem{theorem}{Theorem}[section]
\newtheorem*{theorem*}{Theorem}
\newtheorem*{"theorem"}{``Theorem''}
\newtheorem{corollary}[theorem]{Corollary}

\newtheorem{lemma}[theorem]{Lemma}
\endgroup

\theoremstyle{definition}
\begingroup

\endgroup

\theoremstyle{remark}
\begingroup
\newtheorem{remark}[theorem]{Remark}
\newtheorem{example}[theorem]{Example}
\endgroup 

\numberwithin{equation}{section}
\newcommand\showlabel{\addtocounter{equation}{1}\tag{\theequation}}

\newenvironment{pde}{\left\{\begin{array}{rll} } {\end{array}\right.}

\newcommand{\N}{\mathbb N} 
 
\newcommand{\Z}{\mathbb Z} 
\newcommand{\R}{\mathbb R} 

\newcommand{\dist}{{\rm dist}}

\renewcommand{\div}{{\rm div}}

\newcommand{\spt}{{\mathrm{spt}}}
\newcommand{\id}{\mathrm{id}}

\newcommand{\M}{{\mathcal M}}

\renewcommand{\L}{{\mathcal L}}

\newcommand{\F}{{\mathcal F}}

\newcommand{\B}{\mathcal{B}}
\renewcommand{\P}{\mathbb{P}}

\newcommand{\LRa} {\Leftrightarrow}
\newcommand{\Ra} {\Rightarrow}
\newcommand{\wto}{\rightharpoonup}
\newcommand{\embeds}{\xhookrightarrow{\quad}}

\renewcommand{\d}{\mathrm{d}}

\newcommand{\dx}{\,\mathrm{d}x}
\newcommand{\dy}{\,\mathrm{d}y}
\newcommand{\dz}{\,\mathrm{d}z}
\newcommand{\ds}{\,\mathrm{d}s}
\newcommand{\dt}{\,\mathrm{d}t}

\newcommand{\eps}{\varepsilon}
\newcommand{\average}{{\mathchoice {\kern1ex\vcenter{\hrule height.4pt
width 6pt depth0pt} \kern-9.7pt} {\kern1ex\vcenter{\hrule
height.4pt width 4.3pt depth0pt} \kern-7pt} {} {} }}

\newcommand{\Risk}{\mathcal{R}}
\allowdisplaybreaks
\newcommand{\sw}[1]{\textcolor{black}{#1}}
 
\makeatletter
\@namedef{subjclassname@2020}{2020 Mathematics Subject Classification}
\makeatother
\begin{document}

\title[Barron functions: Representation and pointwise properties]{Representation formulas and pointwise properties for Barron functions}

\author{Weinan E}
\address{Weinan E\\
Department of Mathematics and Program in Applied and Computational Mathematics\\ Princeton University\\ Princeton, NJ 08544\\ USA}
\email{weinan@math.princeton.edu}

\author{Stephan Wojtowytsch}
\address{Stephan Wojtowytsch\\
Princeton University\\
Program in Applied and Computational Mathematics\\
205 Fine Hall - Washington Road\\
Princeton, NJ 08544
}
\email{stephanw@princeton.edu}

\date{\today}


\begin{abstract}
We study the natural function space for infinitely wide two-layer neural networks \sw{with ReLU activation (Barron space)} and establish different representation formulae. In two cases, we describe the space explicitly up to isomorphism. 

Using a convenient representation, we study the pointwise properties of two-layer networks and show that functions whose singular set is fractal or curved (for example distance functions from smooth submanifolds) cannot be represented by infinitely wide two-layer networks with finite path-norm. \sw{We use this structure theorem to show that the only $C^1$-diffeomorphisms which Barron space are affine.}

\sw{
Furthermore, we show that every Barron function can be decomposed as the sum of a bounded and a positively one-homogeneous function and that there exist Barron functions which decay rapidly at infinity and are globally Lebesgue-integrable. This result suggests that two-layer neural networks may be able to approximate a greater variety of functions than commonly believed.
}
\end{abstract}

\maketitle

\setcounter{tocdepth}{1}
\tableofcontents

\section{Introduction}

A {\em two-layer neural network} with $m$ neurons is a function $f:\R^d\to\R$ represented as
\begin{equation}\label{eq finite sum representation}
f(x) = \sum_{i=1}^m a_i\,\sigma(w_i^Tx+ b_i)\qquad\text{or}\qquad f(x) = \frac1m \sum_{i=1}^m a_i\,\sigma(w_i^Tx+ b_i)
\end{equation}
where $\sigma:\R\to\R$ is the (nonlinear) {\em activation function}, $a_i, b_i \in \R$ and $w_i\in \R^d$ are parameters (weights) of the network. In this article, we mostly focus on the case where $\sigma$ is the rectified linear unit (ReLU), i.e. $\sigma(z) = \max\{z,0\}$. We denote the class of all two-layer neural networks with at most $m$ neurons by $\F_m$. Naturally $\F_m$ is not a vector space, but $\F_m+\F_m \subseteq \F_{2m}$. Both the sum and average sum representation induce the same function spaces under natural norm bounds or bounds on the number of parameters.

Under very mild conditions on $\sigma$, any continuous function on a compact set can be approximated arbitrarily well in the uniform topology by two-layer neural networks \cite{cybenko1989approximation, leshno1993multilayer}, i.e.\ the $C^0(K)$-closure of $\F_\infty := \bigcup_{m\in\N} \F_m$ is the entire space $C^0(K)$ for any compact $K\subseteq\R^d$. This result is of fundamental importance to the theory of artificial neural networks, but of little impact in practical applications. In general, the number of neurons $m_\eps$ to approximate a function $f$ to accuracy $\eps>0$ scales like $\eps^{-d}$ and thus quickly becomes unmanageable in high dimension.

On the other hand, Andrew Barron showed in 1993 that there exists a large function class $X$ such that only $O(\eps^{-2})$ neurons are required to approximate $f^*\in X$ to accuracy $\eps$ in $L^2(\P)$ for any Borel probability measure $\P$ \cite{barron1993universal}. Heuristically, this means that non-linear approximation by neural networks, unlike any linear theory, can evade the curse of dimensionality in some situations. The result holds for the same class $X$ for any compactly supported probability measure $\P$ and any continuous {\em sigmoidal} activation function $\sigma$, by which we mean that $\lim_{z\to-\infty} \sigma(z) = 0$ and $\lim_{z\to\infty} \sigma(z) = 1$. It also holds for the nowadays more popular ReLU activation function $\sigma(z) = \max\{z,0\}$ since $\sigma(z+1) - \sigma(z)$ is sigmoidal.

Furthermore, the coefficients of the network representing $f$ can be taken to be bounded on average in the sense that 
\[
\sum_{i=1}^m |a_i| \leq 2C_{f^*}r\qquad\text{or} \qquad  \frac1m\sum_{i=1}^m |a_i| \leq 2C_{f^*}r\qquad\forall\ f^*\in X,
\]
depending on the normalization of the representation. Here $C_f$ is a norm on $X$, $r>0$ is such that the support of $\P$ is contained in $B_r(0)$, and we assume $\sigma$ to be sigmoidal. This result has fundamental significance. In applications, we train the function to approximate values $y_i$ at data points $x_i$ by minimizing an appropriate risk functional. In the simplest case, $y_i = f^*(x_i)$ for a target function $f^*$. If $f(x_i)$ is close to $y_i$ but the coefficients of $f$ are very large, we rely on cancellations between different terms in the sum. Thus $f$ is the difference of two functions (the partial sums for which $a_i >0$/$a_i<0$ respectively) which are potentially several orders of magnitude larger than $f$. Then for any point $x$ which is not one of the data samples $x_i$, $f(x)$ and $f^*(x)$ may be vastly different. We say that $f$ does not generalize well.

The analysis was extended to networks with ReLU activation in \cite{breiman1993hinging}, where such networks are referred to as `hinge functions' because a single neuron activations are given by hyperplanes meeting along a lower-dimensional `hinge'. If $\sigma=$ ReLU, then $\sigma(\lambda z) = \lambda\,\sigma(z)$ for all $\lambda>0$. The homogeneity (and unboundedness) sets ReLU activation and sigmoidal activation apart, and the coefficient bound for ReLU activation is 
 \[
 \sum_{i=1}^m |a_i|\,\big[|w_i| + |b_i|\big] \leq 4C_fr\qquad\text{or}\qquad \frac1m\sum_{i=1}^m |a_i|\,\big[|w_i| + |b_i|\big] \leq 4C_fr.
 \]

The function class $X$ is characterized in \cite{klusowski2018approximation} as $f\in L^1(\R^d)$ such that the Fourier transform $\hat f$ satisfies
\[
C_f:= \int_{\R^d} |\hat f|(\xi)\,\sw{\big[1+|\xi|^2\big]} \d\xi < \infty.
\]
$X$ is a Banach space with norm $\|f\|_X = C_f$. The criterion that $C_f<\infty$ is a non-classical smoothness criterion. If we replaced $|\hat f|$ with $|\hat f|^2$ in the integral, we would obtain the $H^{1/2}$-Sobolev semi-norm. For the weighted $L^1$-norm, the interpretation is not as easy. However, if we multiply by $1 = (1+|x|^{2s})^{1/2}(1+|x|^{2s})^{-1/2}$, use H\"older's inequality and Parseval's identity, we see like in \cite[Section IX, point 15]{barron1993universal} that $C_f\leq c_{d,s}\|f\|_{H^s}$ for $f\in H^{s}(\R^d)$ with $s> \frac{d}2 +2$.

\sw{If $f$ is smooth, but only defined on a suitable compact subset of $\R^d$, we can apply extension results like \cite[Satz 6.10]{dobrowolski2010angewandte} show that $f\in X$. More precisely, if $\Omega$ is a domain in $\R^d$ with smooth boundary $\partial\Omega\in C^{k-1,1}$ for an integer $k> \frac d2+2$, then every $H^k$-function $f$ on $\Omega$ can be extended to a compactly supported $H^k$-function $\bar f$ on $\R^d$ such that $\|\bar f\|_{H^k(\R^d)} \leq C\|f\|_{H^k(\Omega)}$. In particular $C_{\bar f}<\infty$. If we know that elements of $X$ can be approximated efficiently in $L^2(\P)$ with respect to a probability measure $\P$ on $\Omega$, the same is therefore true for $f\in H^k(\Omega)$.} 

On the other hand, all functions $f$ which satisfy $C_f<\infty$ are $C^1$-smooth since ${\partial_i f} (-x) = \widehat{\xi_i\,\hat f(\xi)}$. For ReLU-activated networks, any finite sum of neurons is either linear or non-smooth, and many functions with discontinuous derivatives can be approximated, for example 
\begin{align*}
f(x) &= \max\big\{1-|x|, 0\big\} = \sigma(x-1) - 2\,\sigma(x) + \sigma(x+1),
\end{align*}
is not in $X$ since
\begin{align*}
\hat f(\xi) &= \frac{2-2\,\cos(\xi)}{\sqrt{2\pi}\,\xi^2}.
\end{align*}
 Thus $X$ misses large parts of the approximable function class and $C_f$ may significantly overestimate the number and size of parameters required to approximate a given function. In fact, the criterion is not expected to be sharp since the class $X$ is insensitive to the choice of activation function $\sigma$ and data distribution $\P$. 

In \cite{E:2018ab}, E, Ma and Wu introduced the correct function space for ReLU-activated two-layer neural networks and named it Barron space. It can be seen as the closure of $\F_\infty$ with respect to the {\em path-norm}
\begin{equation}\label{eq path norm}
\|f\|_{\text{path}} = \sum_{i=1}^m |a_i|\,\big[|w_i| + |b_i|\big] \qquad \text{or}\qquad\|f\|_{\text{path}} = \frac1m\sum_{i=1}^m |a_i|\,\big[|w_i| + |b_i|\big]
\end{equation}
instead of the uniform norm, \sw{where $(a_i, w_i, b_i)$ are the non-zero weights of $f\in \F_\infty$ as in \eqref{eq finite sum representation}.}
Further background is given in \cite{weinan2019lei, E:2019aa, approximationarticle}.
A related class of functions is also considered from a different perspective in \cite{bach2017breaking}, where it is referred to as $\F_1$. 
One of the motivation for the present paper is to study these two different perspectives.
\cite{bach2017breaking} uses the signed Radon measure representation and establishes  bounds on Rademacher complexity and generalization gap.
 \cite{E:2018ab,weinan2019lei} characterize the Barron functions using generalized Ridgelet transforms and
  focus on  a priori error bounds for the generalization error.
  Related ideas can also be found in \cite{klusowski2016risk}.

Barron space is the largest function space which is well approximated by two-layer neural networks with appropriately controlled parameters. Target functions outside of Barron space may be increasingly difficult to approximate by even infinitely wide two-layer networks as dimension increases and gradient descent parameter optimization may become very slow in high dimension, see e.g.\ \cite{approximationarticle, dynamic_cod}. A better understanding of the function-spaces associated with classes of neural networks is therefore imperative to understand the function classes which can be approximated efficiently. \sw{We describe Banach spaces associated to multi-layer networks in \cite{deep_barron}.}

 
In this article, we provide a comprehensive view of the pointwise and functional analytic properties of ReLU-Barron functions in the case of two-layer networks. We discuss different representations, the structural properties of their singular set, and their behavior at infinity. In heuristic terms, we show that the singular set of a Barron function is a countable union of affine subspaces of dimension $\leq d-1$ of $\R^d$. 

\sw{
Understanding which functions are and are not in Barron space is a first step towards understanding which kind of functions can be approximated efficiently by two-layer neural networks while maintaining meaningful statistical learning bounds. We demonstrate that contrary to popular belief, this includes functions which decay rapidly at infinity, while providing} an easy to check criterion that a Lipschitz function cannot be represented as a two-layer network. For example, the distance function from a curved $k$-dimensional manifold in $\R^d$ is not in Barron space. \sw{Understanding which functions can be approximated by neural networks of a given depth without the curse of dimensionality is essential to choose the simplest sufficient neural network architecture for a given purpose. Choosing the simplest sufficient network model can significantly reduce the difficulty and energy consumption of training.}

Some previous works approach the infinite neuron limit of two-layer neural networks from the perspective of statistical mechanics where neurons are viewed as exchangeable particles accessed mostly through their distribution \cite{chizat2018global, mei2018mean,rotskoff2018neural,sirignano2018mean}. In a part of this work, we present an alternative description in which the particles are indexed, as is the case in practical applications. The approach is conceptually easy and convenient from the perspective of gradient flow training, but not suited for variational arguments. The mean field gradient flow of neural networks is described by a Banach-space valued ODE in this setting rather than a PDE like in the usual picture. A similar approach was developed for the mean field dynamics of multi-layer networks in \cite{nguyen2020rigorous} under the name of `neuronal embeddings'. In the language of that article, we show that every Barron function can be represented via a neuronal embedding, whereas \cite{nguyen2020rigorous} focusses on the training of parameters from a given initial distribution.

Furthermore, we explore their relationship to classical function spaces in terms of embeddings, and to two-layer neural networks in terms of direct and inverse approximation theorems. To provide a comprehensive view of a relatively new class of function spaces, parts of this article review material from previous publications in a more functional-analytic fashion.

The article is structured as follows. In Section \ref{section representations}, we describe different ways to describe functions in Barron space, which are convenient for dynamic or variational purposes respectively, or philosophically interesting. Following Section \ref{section representations}, only the representation of Section \ref{section signed measure} is used in this work. Section \ref{section properties} is devoted to the relationship of Barron space to classical function spaces on the one hand and to finite two-layer neural networks on the other. In two special cases, we characterize Barron space exactly in Section \ref{section special cases}. We conclude by establishing some structural and pointwise properties for Barron function in Section \ref{section structure}. 

\subsection{Notation}
If $X, Y$ are measurable space, $f:X\to Y$ is measurable and $\mu$ is a measure on $X$, then we denote by $f_\sharp\mu$ the push-forward measure on $Y$, i.e.\ $f_\sharp\mu(A) = \mu(f^{-1}(A))$. If $\mu$ is a measure and $\rho \in L^1_{loc}\mu$, we denote by $\rho\cdot\mu$ the measure which has density $\rho$ with respect to $\mu$. 

All measures will be assumed to be finite and Borel. Since all spaces considered are finite-dimensional vector spaces or manifolds, they are in particular Polish and thus therefore all measures considered below are Radon measures. 

The norm on the space of Radon measures is $\|\mu\| = \sup_{U,V} \mu(U) - \mu(V)$ where $U, V$ are measurable sets. On compact subsets of $\R^d$, the norm is induced by duality with the space of continuous functions. We observe that $\|f_\sharp\mu\|\leq \|\mu\|$ in general and $\|f_\sharp\mu\| = \|\mu\|$ if $\mu$ is non-negative.

See \cite{MR2759829} and \cite{evans2015measure} for background information and further terminology in functional analysis and measure theory respectively.

\section{Different representations of Barron functions}\label{two-layer representations}\label{section representations}

There are many equivalent ways to represent Barron functions with different advantages in different situations. In this section, we discuss eight of them, some of which have previously been considered in \cite{E:2018ab,weinan2019lei}. The main novel contributions of this article are collected in Sections \ref{section signed measure}, \ref{section indexed particles} and \ref{section indexed particles new}. 

To simplify notation, we identify $x\in \R^d$ with $(x,1)\in \R^{d+1}$ and abbreviate $(w,b)\in\R^{d+1}$ as $w$. In particular, by an abuse of notation, $w^Tx = w^Tx+b$.

Let $\P$ be a probability measure on $\R^d$ (or $\R^d\times\{1\}$ respectively). We will refer to $\P$ as the data distribution and assume that $\sw\P$ has finite first moments. We assume that we are given a norm $|\cdot|$ on data space $\R^{d+1}$ (i.e.\ in the $x$-variables) and consider the dual norm (also denoted by $|\cdot|$) on $\R^{d+1}$ for the $w$-variables such that $|w^Tx| \leq |w|\,|x|$. It will be obvious from context which norm is used where. Usually, we imagine that $|\cdot| = |\cdot|_{\ell^2}$ is the Euclidean norm on both data and parameter space or that $|x| = |x|_{\ell^\infty}$ and $|w| = |w|_{\ell^1}$, but the analysis only depends on duality, not the exact pairing.

\subsection{Representation by parameter distribution}\label{section parameter distribution}

A two-layer network with $m$ neurons can be written as a normalized sum
\[
f(x) = \frac1m\sum_{i=1}^m a_i\,\sigma(w_i^Tx) = \int a\,\sigma(w^Tx)\,\pi_m(\d a\otimes \d w)
\]
where $\pi_m = \frac1m\sum_{i=1}^m \delta_{(a_i, w_i)}$ is the empirical parameter distribution of the network. A natural way to extend this to infinitely wide networks is to allow {\em any} parameter distribution $\pi$ on the right hand side. For a general Radon probability measure $\pi$ on $\R\times \R^{d+1}$, we set
\[
f_\pi(x) = \int_{\R\times \R^{d+1}} a\,\sigma(w^Tx)\,\pi(\d a\otimes \d w).
\]
The parameter distribution $\pi$ representing a function $f$ is never unique since $f_\pi \equiv 0$ for any $\pi$ which is invariant under the reflection $T(a,w) = (-a, w)$. For ReLU activation, a further degeneracy stems from the fact that $z = \sigma(z) - \sigma(-z)$ for any $z\in \R$ and thus
\[
0 = x+\alpha - x - \alpha = \sigma(x+\alpha) - \sigma\big(-(x+\alpha)\big) - \sigma(x) + \sigma(-x) - \sigma(\alpha) + \sigma(-\alpha)\qquad\forall\ \alpha\in\R.
\]
Finally, we list a two-dimensional degeneracy. Recall that we can represent 
\[
x^2 = 2\int_\R 1_{\{t>0\}}\sigma(x-t)\dt\qquad\text{for }x>0.
\]
In two dimensions, the function $f(x_1, x_2) = x_1^2 + x_2^2$ on the unit disk can therefore be represented by a parameter distribution which is concentrated on the coordinate-axes. Due to rotational invariance, the same is true for any other orthonormal basis, or the parameters could be chosen in a rotationally symmetric fashion. 

The Barron norm is the generalization of the path-norm in \eqref{eq path norm}. To compensate for the non-uniqueness in representation, we define
\[
\|f\|_{\B(\P)}= \inf\left\{\int_{\R\times \R^{d+1}}|a|\,|w|\,\pi(\d a\otimes \d w)\:\bigg|\:\pi\text{ Radon probability measure s.t. }f_\pi = f\text{ $\P$-a.e.}\right\}.
\]
The infimum of the empty set is considered as $+\infty$. Clearly
\[
|f_\pi(x) - f_\pi(y)| \leq \int_{\R^{d+2}}|a|\,|w^T(x-y)|\,\pi(\d a\otimes \d w) \leq \|f\|_{\B(\P)}|x-y|,
\]
so in particular $f_\pi$ grows at most linearly at infinity. Since $\P$ has finite first moments, this means that $f_\pi \in L^1(\P)$. We introduce Barron space as
\[
\B(\P) = \{f \in L^1(\P)\::\:\|f\|_{\B(\P)} < \infty\}.
\]
Approaching Barron space through the parameter distribution $\pi$ is natural from the point of view that we know the parameters $(a_i, w_i)$ in applications better than the induced function. It is also useful, especially when considering dynamics. Namely, let $\Theta= (a_i, w_i)_{i=1}^m$ and 
\begin{equation}\label{eq f Theta}
f_\Theta(x) = \frac1m \sum_{i=1}^m a_i\,\sigma(w_i^Tx).
\end{equation}
Then the parameters $\Theta$ evolve by the Euclidean gradient flow of a risk functional
\begin{equation}\label{eq finite risk}
\Risk(\Theta) = \int_{\R^d\times \{1\}} |f_\Theta - f^*|^2(x)\,\P(\d x)
\end{equation}
if and only if their empirical distribution evolves by the 2-Wasserstein gradient flow of the extended risk functional
\[\showlabel\label{eq extended risk}
\Risk(\pi) = \int_{\R^d\times \{1\}} |f_\pi - f^*|^2(x)\,\P(\d x)
\]
(up to time rescaling). \sw{The heuristic reason behind this connection is that the Wasserstein-distance is `horizontal' and that measure is `transported' along curves in optimal transport distances rather than `teleported' as in `vertical' distances like $L^2$. In a more mathematically precise fashion, the `particles' $(a_i, w_i)$ follow trajectories which can be viewed as the characteristics of a continuity equation 
\[
\dot \rho = \div (\rho \,\nabla V)
\]
with feedback between the particles and the transport vector field $\nabla V$. The connection stems from the observation that the gradient of the risk functional in \eqref{eq finite risk} is given by
\begin{align*}
\nabla_{(a_i, w_i)} \Risk(\Theta)
	&= \frac{2}m \nabla_{(a_i,w_i)} \int_{\R^d} \left(\frac1m\sum_{i=1}^m a_j\,\sigma(w_j^Tx) - f^*(x)\right)a_i\,\sigma(w_i^Tx)\,\P(\d x)
\end{align*}
which coincides with the gradient of a potential, evaluated at the position of the particle: 
\[
\nabla_{(a_i, w_i)} \Risk(\Theta) = \frac2m\,\nabla V(a_i,w_i; \Theta), \qquad V(a,w; \Theta) = \int_{\R^d} \big(f_\Theta-f^*\big)(x)\,a\,\sigma(w^Tx)\,\P(\d x).
\]
The passage to the limit 
\[
V(a,w; \pi) = \int_{\R^d} \big(f_\pi-f^*\big)(x)\,a\,\sigma(w^Tx)\,\P(\d x)
\]
is easy on the formal level, and eliminating the factor $\frac1m$ merely corresponds to a rescaling of time.
 The link between Wasserstein gradient flows and continuity equations has been exposed first in the seminal article \cite{jordan1998variational}. For details in the context of machine learning}, see \cite[Proposition B.1]{chizat2018global} or \cite[Appendix A]{relutraining}. The result is not specific to $L^2$-risk and holds much more generally. The Wasserstein-distance is computed with respect to the Euclidean distance on parameter space here. 

The parameter distribution picture is available for general two-layer networks regardless of the activation function.

\begin{remark}
In practice, the weights $(a_i, w_i)_{i=1}^m$ of a neural network are initialized randomly according to a distribution $\pi^0$ in such a way that $(a_i, w_i)$ and $(-a_i, w_i)$ are equally likely. Such an initialization gives the network the flexibility to develop features in all relevant directions during training. In the continuum limit, $f_\pi \equiv 0$ at initial time, but 
\[
0< \int_{\R^{d+2}}|a|\,|w|\,\pi(\d a \otimes \d w).
\]
The upper bound for the Barron norm
\[
\|f_\Theta\|_{\B(\P)} \leq \frac1m \sum_{i=1}^m |a_i|\,|w_i|
\]
is therefore easy to compute, but not assumed to be particularly tight (at least in the infinite width limit).
\end{remark}

\subsection{Spherical graph representation}
By the positive $1$-homogeneity of the ReLU activation we have 
\begin{align*}
f_\pi(x) &= \int_{\R\times \{w\neq 0\}} a\,|w|\,\sigma\left(\frac{w}{|w|}^Tx\right)\,\pi(\d a\otimes \d w)\\
	&= \int_{\R\times S^d} \tilde a\, \sigma(\tilde w^Tx)\,\tilde \pi(\d \tilde a \otimes \d \tilde w)
\end{align*}
where $\tilde \pi = T_\sharp \pi$ is the push-forward of $\pi$ along the map $(a,w) \mapsto (a\,|w|, w/|w|)$. Since $\tilde \pi$ is a Radon measure, we can apply \cite[Theorem 4.2.4]{attouch2014variational} to decompose $\tilde \pi$ into a marginal $\hat\pi$ on $S^d$ and conditional probabilities $\pi^w$ on $\R$ such that
\[
\int_{\R\times S^d} f(\tilde a, \tilde w) \,\tilde\pi(\d \tilde a\otimes \d \tilde w) = \int_{S^d} \left(\int_\R f(a,w)\,\pi^w(\d a)\right)\,\hat \pi(\d w)
\]
for every $\tilde\pi$-measurable function $f$. In particular, the function
\[
w\mapsto \int_\R f(a,w)\,\pi^w(\d a)
\]
is $\hat\pi$-measurable. For a neural network function $f(a,w) = a\,\sigma(w^Tx)$, we find that
\begin{align*}
f_\pi(x) &= \int_{S^d}\left(\int_\R a\,\pi^w(\d a)\right)\,\sigma(w^Tx)\,\hat\pi(\d w)\\
	&=: \int_{S^d} \hat a(w)\,\sigma(w^Tx)\,\hat \pi(\d w)
\end{align*}
with
\[
\hat a(w) = \int_\R a\,\pi^w(\d a).
\]
We have thus written $f_\pi$ as a graph over the unit sphere, which we denote by $f_{\hat\pi,\hat a}$. \sw{Since
\begin{align*}
\|\hat a\|_{L^1(\hat\pi)} &= \int_{S^d} \big|\hat a(w)\big|\,\hat \pi(\d w) \\
	&\leq \int_{S^d}\left|\int_\R a\,\pi^w(\d a)\right|\,\hat\pi(\d w)\\
	&\leq \int_{S^d}\int_\R |a|\,\pi^w(\d a)\,\hat\pi(\d w)\\
	&= \int_{S^d} |a|\,\tilde\pi(\d a\otimes \d w)\\
	&= \int_{S^d} |a|\,|w|\,\pi(\d a\otimes \d w),
\end{align*}
we note that 
\[
\inf_{(\hat a, \hat \pi) \text{ s.t. } f_{\hat\pi, \hat a} = f\text{ $\P$-a.e.}} \|\hat a\|_{L^1(\hat\pi)} \leq \inf_{\pi \text{ s.t. } f=f_\pi\text{ $\P$-a.e.}} \int_{\R\times \R^{d+1}}|a|\,|w|\,\pi(\d a\otimes \d w)
\]
The inverse inequality is obtained by considering the distribution $\pi = \psi_\sharp \hat\pi$ where $\psi(w) = (\hat a(w), w)$ which satisfies
\[
\int_{\R^{d+2}} a\,\sigma(w^Tx)\,\pi(\d a \otimes \d w) = \int_{S^d} \hat a(w)\,\sigma(w^Tx)\,\hat\pi(\d w), \qquad \int_{\R^{d+2}} |a|\,|w|\,\pi(\d a \otimes \d w) = \int_{S^d} |\hat a(w)|\,\hat\pi(\d w)
\]
by the definition of the push-forward.}
Taking the infimum, we find that
\[
\|f\|_{\B(\P)} = \inf\left\{\|\hat a\|_{L^1(\hat\pi)}\:\bigg|\:\hat\pi \text{ Radon probability measure on }S^d, \:\hat a\in L^1(\hat\pi), \:f= f_{\hat \pi, \hat a} \:\P-\text{a.e.}\right\}.
\]

\begin{remark}\label{remark graph Lp}
Without loss of generality, we can absorb all variation into the measure $\hat\pi$ and have $|\hat a| = \|f\|_{\B(\P)}$ almost everywhere. More specifically, note that $f_{\hat a, \hat \pi} = f_{\hat a/\rho,\, \rho\cdot\hat\pi}$ for any function $\rho$ such that $\rho>0$ if $\hat a>0$ and
\[
\int_{S^d}\rho(w)\,\hat \pi(\d w) = 1.
\]
We specify 
\[
\rho = \frac{|\hat a|}{\|\hat a\|_{L^1(\hat \pi)}}\qquad \Ra\qquad \frac{\hat a}{\rho} = \mathrm{sign}(\hat a)\,\|\hat a\|_{L^1(\hat \pi)}.
\]
In particular, the Barron norm can equivalently be written as $\inf_{\hat\pi,\hat a} \|\hat a\|_{L^p(\hat \pi)}$ for any $p\in [1,\infty]$.
\end{remark}

The spherical graph representation is specific to positively homogeneous activation functions. It is not per se useful to us directly, but it provides a link to the representation of Barron functions by signed measures (Section \ref{section signed measure}). Note that different distributions $\pi$ may give rise to the amplitude function $\hat a$ and spherical measure $\hat \pi$, so the link to dynamics through Wasserstein gradient flows is lost in this description and all following ones that are derived from it.

\subsection{Signed measure on the sphere}\label{section signed measure}

As in Remark \ref{remark graph Lp}, all relevant information about the tuple $(\hat a, \hat\pi)$ is contained in the signed measure $\hat \mu = \hat a \cdot\hat\pi$. We set
\[
f_{\hat \mu}(x) = \int_{S^d} \sigma(w^Tx)\,\hat\mu(\d w)
\]
and observe that \sw{
\[
\|\hat\mu\|_\M = \int_{S^d} 1\,|\hat \mu|(\d w) = \int_{S^d} |\hat a(w)|\,\hat \pi(\d w).
\]
Here $\|\cdot\|_\M$ is the total variation norm on the space of Radon measures.
Taking the infimum first on the left and then on the right, we find that
\[
\inf_{\hat\mu\text{ s.t. }f_{\hat \mu} = f \text{ $\P$-a.e.}}\|\hat\mu\|_\M \leq \|f\|_{\B(\P)}.
\]
On the other hand, given a signed Radon measure $\hat \mu\neq 0$ on $S^d$, we define
\[
\hat \pi:= \frac{|\hat \mu|}{\|\hat\mu\|_\M}, \qquad \hat a = \frac{\d \hat\mu}{\d \hat \pi}
\]
where the Radon-Nikodym derivative of $\hat \mu$ with respect to $\hat\pi$ is well-defined since both parts $\hat\mu_\pm$ of the Hahn-decomposition of $\hat \mu$ are absolutely continuous with respect to $|\hat \mu|$ -- see e.g.\ \cite[Sections 7.4 and 7.5]{klenke2006wahrscheinlichkeitstheorie} for the relevant definitions and properties. Then
\[
\hat \mu(U) = \int_U \hat a(w)\,\hat \pi(\d w), \qquad \int_{S^d} f(\hat w) \,\hat\mu(\d w) = \int_{S^d} f(\hat w)\,\hat a(w)\,\hat \pi(\d w)
\]
for all measurable sets $U\subseteq S^d$ and all measurable functions $f:S^d\to\R$. In particular
\[
f_{\hat\pi,\hat a}(x) = \int_{S^d}\sigma(w^Tx)\,\hat a(w)\,\hat \pi (\d w)  = \int_{S^d}\sigma(w^Tx)\,\hat\mu(\d w) = f_{\hat \mu}(x)
\]
for all $x$ and
\begin{align*}
\int_{S^d} |\hat a(w)|\,\hat \pi(\d w) &=\int_{S^d}1\,|\hat\mu|(\d w) = \|\hat \mu\|_{\M}.
\end{align*}
Taking the infimum on the left shows that 
\[
\|f\|_{\B(\P)} \leq \|\hat\mu\|_\M
\]
for any admissible $\hat\mu$. As a consequence
}
\[
\|f\|_{\B(\P)} = \inf\left\{\|\hat\mu\|_\M\::\: \hat\mu\text{ signed Radon measure on }S^d,\:f= f_{\hat \mu}\:\P-\text{a.e.}\right\}.
\]

This perspective is particularly convenient with an eye towards variational analysis. Compactness results in the space of Radon measures are much stronger here since we can restrict ourselves to the {\em compact} parameter space $S^{d+1}$. Barron space is isometric to the quotient of the space of Radon measures on the sphere $\M$ by the closed subspace
\[
N_\P:= \{\hat\mu \in \M\:|\: f_{\hat \mu} =0 \:\P-\text{a.e.}\}.
\]
Thus this perspective establishes an otherwise nontrivial result automatically.

\begin{theorem}\label{theorem Barron Banach}
$\B(\P) \widetilde = \M/N_\P$ is a Banach space.
\end{theorem}

On the other hand, the link to gradient flow dynamics is lost in this picture. Directly optimizing the measure $\mu_m = \sum_{i=1}^m a_i\,\delta_{w_i}$ rather than the weights $(a_i, w_i)_{i=1}^m$ was considered in \cite{bach2017breaking} (for more general activation functions with homogeneity $\alpha\geq 0$) and found to be computationally unfeasible. \sw{The advantage of this perspective on optimization is that the map $\mu\to f_\mu$ is linear, so common risk functionals are convex. The disadvantage is that optimization in a space of Radon measures is difficult in practice.}

Also this perspective is most useful for homogeneous activation functions.

\subsection{Signed measure on parameter space}
We can generalize the parameter distribution representation of Section \ref{section parameter distribution} by allowing general signed Radon measures $\mu$ in the place of $\rho$, i.e.
\[
f_\mu(x) = \int_{\R\times \R^{d+1}} a\,\sigma(w^Tx)\,\mu(\d a\otimes \d w).
\]
Unlike in Section \ref{section signed measure}, $\mu$ is a signed measure on the whole space $\R^{d+2}$ here. This representation does not rely on the homogeneity of ReLU activation.
\sw{We define a norm}
\[
\|f\|'_\P= \inf_{\{\mu|f_\mu =f\text{ $\P$-a.e.}\}} \int_{\R\times \R^{d+1}}|a|\,|w|\,|\mu|(\d a\otimes \d w)
\]
where $|\mu| = \mu^+ + \mu^-$ is the total variation measure of $\mu$. \sw{It is immediate to see that $\|f\|_\P'\leq \|f\|_{\B(\P)}$ by comparison with the representation for a signed measure on the sphere where $\mu$ is restricted to the set $a=|w|=1$. We prove the opposite inequality by comparing to the parameter distribution representation.}

For $\lambda\in \R$, denote $T^\lambda:\R\times \R^{d+1}\to \R\times \R^{d+1}$, $T^\lambda(a,w) = (\lambda a, w)$. For a map $\psi:X\to Y$ between sets and a signed measure $\nu$ on $X$, denote by $\psi_\sharp\nu$ the push-forward measure on $Y$. Note that $\|\psi_\sharp\nu\|\leq \|\nu\|$ and that equality holds for positive measures. \sw{Thus}
\[
\pi:= \frac12 \left[ \frac1{\|\mu^+\|} T^{2\|\mu^+\|}_\sharp\mu_+ + \frac1{\|\mu^-\|}T^{-2\|\mu^-\|}_\sharp\mu_-\right]
\]
satisfies $f_\pi = f_\mu$ and 
\[
\int_{\R\times \R^{d+1}}|a|\,|w|\,\pi(\d a\otimes \d w) = \int_{\R\times \R^{d+1}}|a|\,|w|\,|\mu|(\d a\otimes \d w).
\]
\sw{Taking the infimum first on the left and then on the right, we obtain the inverse inequality $\|f\|_{\B(\P)}\leq \|f\|_\P'$.}

\subsection{Indexed particle perspective I}\label{section indexed particles}
All representations of two-layer networks discussed above were invariant under the natural symmetry
\[
f(x) = \frac1m \sum_{i=1}^m a_{s_i}\,\sigma(w_{s_i}^Tx)
\]
where $s\in S_m$ is a permutation of the indices. We say that the particles are {\em exchangable}. Nevertheless, in all practical applications particles $(a,w)$ are indexed by $i\in \{1,\dots, m\}$. We now develop a parametrized perspective of neural networks. Note that
\[
f(x) = \frac1m \sum_{i=1}^m a_i\,\sigma(w_i^Tx) = \int_0^1 a_\theta\,\sigma(w^T_\theta x)\,\d\theta
\]
where $a_\theta = a_k$ and $w_\theta = w_k$ for $\frac {k-1}m \leq \theta < \frac km$. 
Using scaling invariance on finite networks, we may assume that $|w|\equiv 1$ and obtain a uniform $L^1$-bound on $a$.
More generally, for $a\in L^1(0,1)$ and $w\in L^\infty\big((0,1); S^d\big)$ (or $a, w\in L^2$) we define 
\[
f_{(a,w)} (x) = \int_0^1 a_\theta\,\sigma(w_\theta^Tx)\,\d\theta
\]
and
\[
\|f\|_{\B'(\P)} = \inf_{\{(a,w):f=f_{a,w}\}} \int_0^1|a_\theta|\,|w_\theta|\,\d\theta, \qquad \B'(\P) = \{f\in C^{0,1}_{loc}(\R^d) : \|f\|_{\B'(\P)}<\infty\}.
\]

This perspective is fundamentally different from the previous ones, and it is not immediately clear whether the spaces $\B(\P)$ and $\B'(\P)$ coincide. We prove this as follows.

Assume that $f\in \B'(\P)$. Then
\[
f(x) = \int_0^1 \overline a_\theta \,\sigma(\overline w^T_\theta x)\,\d\theta =  \int_{\R\times \R^{d+1}} a\,\sigma(w^Tx)\,\big((\bar a,\bar w)_\sharp\L^1|_{(0,1)}\big)(\d a\otimes \d w)
\]
where $(\bar a,\bar w)_\sharp\L^1|_{(0,1)}$ denotes the push-forward of one-dimensional Lebesgue measure on the unit interval along the map $(\bar a, \bar w): (0,1)\to \R^{d+2}$. Thus $\B'(\P)$ is a subspace of $\B(\P)$. 

Before we prove the opposite inclusion, we recall an auxiliary result.

\begin{lemma}\label{lemma technical measure}
\begin{enumerate}
\item There exists a bijective measurable map $\phi :[0,1]^d\to [0,1]$.
\item Let $\bar\pi$ be any probability measure on $\R$. Then there exists a measurable map $\psi:[0,1]\to \R$ such that $\bar\pi = \psi_\sharp \L^1$.
\end{enumerate}
\end{lemma}

\begin{proof}
{\bf Claim 1.} For $1\leq i \leq d$, we can write $x_i = \sum_{k=1}^\infty \alpha^i_k 10^{-k}$ with $\alpha^i_k \in \{0, \dots,9\}$. The map $a^i_k: Q\to \{0,\dots,9\}$ $a^i_k(x)= \alpha^i_k$ satisfies
\[
(a^i_k)^{-1}(\{\alpha\}) = [0,1]^{i-1}\times \bigcup_{\beta_1, \dots, \beta_{k-1} \in \{0, \dots, 9\}} \left[\sum_{j=1}^{k-1} \beta_j 10^{-j} + \alpha\,10^{-k}, \:\sum_{j=1}^{k-1} \beta_j 10^{-j} + (\alpha+1)\,10^{-k}\right) \times [0,1]^{d-i}
\]
and is therefore measurable. Thus also the maps
\[
\phi_m:Q\to [0,1], \qquad \phi_m(x) = \sum_{k=0}^m \sum_{i=1}^d \alpha_{k+1}^i(x)\,10^{-(kd+i)}
\]
and their pointwise limit 
\[
\phi:Q\to [0,1], \qquad \phi(x) = \sum_{k=0}^\infty \sum_{i=1}^d \alpha_{k+1}^i(x)\,10^{-(kd+i)}
\]
are measurable. They are also bijective since each point is represented uniquely by its decimal representation (since we excluded trailing $9$s). If $\phi(x) = \phi(y)$, then all coordinates of $x$ and $y$ have the same decimal expansion and thus are the same point. On the other hand, for $z\in [0,1]$, it is easy to define $x= \phi^{-1}(z)$.

{\bf Claim 2.} This is a well-known result in probability theory and used in numerical implementations to create random samples from distributions by drawing a random sample from the uniform distribution on $(0,1)$ and applying a suitable transformation. The map $\psi =\chi^{-1}$ for $\chi:\R\to [0,1]$, $\chi(z) = \bar\pi(-\infty,z]$ satisfies the conditions. $\chi$ is monotone increasing, but usually not strictly. In this case, we choose the left-continuous version of the derivative. For details, see e.g.\ \cite[Satz 1.104]{klenke2006wahrscheinlichkeitstheorie} and its proof.
\end{proof}

Now assume that $f\in \B(\P)$. Then we can describe $f$ as a spherical graph, i.e.\ $f(x) = \int_{S^d} a(w)\,\sigma(w^Tx)\,\pi(\d w)$ for a probability measure $\pi$ and an amplitude function $a\in L^1(\pi)$. We may assume that $a$ is defined on the whole space (e.g.\ by $a\equiv 0$ outside $S^d$). Denote $\tilde \phi = \phi \circ [\frac12(\cdot + 1)] : [-1,1]^d\to [0,1]$ and 
 $\bar\pi= \tilde\phi _\sharp \pi$,  $\hat w:[0,1]\to Q$, $\hat w_\theta = \tilde\phi^{-1}(\theta)$. By definition we have
\begin{align*}
\int_{[0,1]} a(\hat w_\theta)\,\sigma(\hat w_\theta x)\,\bar \pi(\d\theta) &= \int_{[-1,1]^d}a(w)\,\sigma(w^Tx)\,\pi(\d w)\\
	&= \int_{S^d}a(w)\,\sigma(w^Tx)\,\pi(\d w).
\end{align*}
The measure $\hat\pi$ is highly concentrated and we use the second claim from Lemma \ref{lemma technical measure} to normalize it. Namely, take $\psi:[0,1]\to [0,1]$ as described. Then
\begin{align*}
f(x) &= \int_{[0,1]} a(\hat w_\theta)\,\sigma(\hat w_\theta x)\,\bar \pi(\d\theta)\\
	&= \int_{[0,1]} a(\hat w_\theta)\,\sigma(\hat w_\theta x)\,\psi_\sharp\L^1(\d\theta)\\
	&= \int_0^1 a(\hat w_{\psi(\theta)})\,\sigma(\hat w_{\psi(\theta)}^Tx)\,\d\theta.
\end{align*} 
In particular, we note that $\hat w_{\psi(\theta)} = (\tilde\phi^{-1}\circ\psi)(\theta) \in S^d$ almost surely and set $a_\theta = a(\tilde\phi^{-1}\circ\psi(\theta))$, $w_\theta = \tilde\phi^{-1}\circ \psi(\theta)$. We thus find that $\B(\P) \subseteq \B'(\P)$. The same argument implies that $\|\cdot\|_{\B(\P)} = \|\cdot \|_{\B'(\P)}$.

\begin{remark}
Similarly as in Remark \ref{remark graph Lp}, we can reparametrize the maps $a, w$ by 
\[
\tilde a_\theta = a_{\rho(\theta)}\,\rho'(\theta), \qquad \tilde w _\theta = w_{\rho(\theta)}
\]
for any diffeomorphism $\rho:(0,1)\to(0,1)$ and in particular achieve that $|\tilde a|$ is constant.
\end{remark}

\begin{remark}
While elementary, the construction made use of highly discontinuous measurable maps and reparametrizations. If we allowed general probability measures on $[0,1]$, we could fix the map $w$ instead to be a (H\"older-continuous) space-filling curve in $S^d$. 
\end{remark}

The indexed particle representation is easy to understand, but has clear drawbacks from the variational perspective. The norm does not control the regularity of the map $w$, which means that at most, we obtain weak compactness for $w$ under norm bounds. After applying $\sigma$, we cannot pass to the limit in $f_{a_n, w_n}$ even in weak norms, and variational results like the Inverse Approximation Theorem \ref{inverse approximation theorem} cannot be obtained in the indexed particle representation. Much like the parameter distribution representation, indexed particles are on the other hand convenient from a dynamic perspective.

\begin{lemma}\label{lemma evolution two layer}
Let $f(x) = \frac1m \sum_{k=1}^m a_k\,\sigma(w_k^Tx)$ where the parameters $ \Theta = \{a_k, w_k\}_{k=1}^m$ evolve under the time-rescaled gradient flow  
\[
\dot \Theta = -m\,\nabla\Risk(\Theta),\qquad \Risk(\Theta) = \int \ell(f_\Theta(x),y)\,\P(\d x\otimes \d y)
\]
for a sufficiently smooth and convex loss function $\ell$. Then the functions 
\[
a(t,\theta) = a_k(t)\text{ and }w(t,\theta) = w_k(t)\quad  \text{for }\frac {k-1}m \leq \theta < \frac km
\]
evolve by the $L^2$-gradient flow of
\[
\Risk(a,w) = \int \ell(f_{(a,w)}(x),y)\,\P(\d x\otimes \d y).
\]
\end{lemma}

We assume that the gradient flow for finitely many parameters exists. This can be established by the Picard-Lindel\"off theorem if $\sigma$ is sufficiently smooth. For ReLU activation, existence of a classical gradient flow is guaranteed if $\P$ is a suitable population risk measure, see \cite{relutraining}.

\begin{proof}
All functions lie in $L^2$ for all times since they are given by finite step functions. Let $\psi\in L^2[0,1]$. We compute that
\begin{align*}
\frac{d}{d\eps}\bigg|_{t=0} \Risk(a+ \eps\psi, w) &= \int (\partial_1\ell)(f_{(a,w)}(x),y)\,\frac{d}{d\eps}\bigg|_{\eps =0} f_{(a+\eps\psi,w)}\P(\d x\otimes \d y)\\
	&= \int (\partial_1\ell)(f_{(a,w)}(x),y)\, f_{(\psi,w)}\P(\d x\otimes \d y)\\
	&= \int \left(\int (\partial_1\ell)(f_{(a,w)}(x),y)\,\sigma(w_\theta^Tx) P(\d x\otimes \d y)\right)\,\psi(\theta)\,\d\theta \showlabel\label{eq variational gradient}\\
\delta_a \Risk(a,w) &= \int (\partial_1\ell)(f_{(a,w)}(x),y)\,\sigma(w_\theta^Tx) P(\d x\otimes \d y)
\end{align*}
since $f_{(a,w)}$ depends on $a$ linearly. For $\frac{k-1}m \leq \theta < \frac km$, this is precisely $m$ times the gradient of $\Risk(\Theta)$ with respect to $a_k$. The same result holds for $w_k$. The key point is that the gradient flow is defined entirely pointwise and the only interaction between $a(\theta_1)$, $a(\theta_2)$ for $\theta_1\neq \theta_2$ (or $w$, etc.) is through the function $f_{(a,w)}$.
\end{proof}

The gradient flow has no smoothing effect and preserves step functions for all time. Note that the normalization $|w|\equiv 1$ is not preserved under the gradient flow.

\subsection{Indexed particle perspective II}\label{section indexed particles new}

We chose the unit interval $(0,1)$ equipped with Lebesgue measure as an `index space' for particles and showed that it is expressive enough to support any Barron function. We also demonstrated that this perspective can be linked to neural network training. In this section, we sketch a different indexed particle approach where the index space may depend on the Barron function, but gradient flow training can be incorporated in a very natural fashion.

Let $\pi^0$ be a probability distribution on $\R^{d+2}$ and $(\bar a, \bar w):\R^{d+2}\to \R\times \R^{d+1}$ measurable functions. Then we can define 
\[
f_{\pi^0; \bar a, \bar w}(x) = \int_{\R^{d+2}}\bar a(a,w)\,\sigma\big(\bar w(a,w)^Tx\big)\,\pi^0(\d a\otimes \d w),
\]
i.e.\ we consider particles $(\bar a, \bar w)$ indexed by $(a,w)$. If $\bar a(a,w) = a$ and $\bar w(a,w)= w$, this merely recovers the parameter distribution perspective. While uninteresting from the statical perspective of function representation, it allows a different view of gradient flow training. Namely, for fixed $\pi^0$ we can consider the following ODEs in $L^2(\pi^0;\R^{d+2})$:
\[\showlabel\label{eq ode version gradient flow}
\begin{pde}
\frac{d}{dt}\big({\bar a}, \bar w\big)(a,w; t) &= - \nabla_{(\bar a, \bar w)}\Risk\big(\bar a(t), \bar w(t)\big) &t>0\\ (\bar a, \bar w) &= (a,w) &t=0\end{pde}
\]
where
\[
\Risk(\bar a, \bar w) = \int_{\R^d} \big|f_{\pi^0; \bar a, \bar w}- f^*\big|^2(x)\,\P(\d x)
\]
and $\nabla_{(\bar a, \bar w)}$ describes the variational gradient
\[
\big(\nabla_{(\bar a, \bar w)} \Risk(\bar a, \bar w)\big)(a,w) = \int_{\R^{d+2}}\big(f_{\pi^0; \bar a, \bar w}- f^*\big)(x)\,\begin{pmatrix} \sigma\big( \bar w(a,b)^Tx\big)\\ \bar a(a,w)\,\sigma'(\bar w(a,b)^Tx\big)x\end{pmatrix}\,\P(\d x)
\]
analogous to \eqref{eq variational gradient}.

\begin{lemma}\label{lemma equivalent flows}
Let $\pi^0$ be a probability distribution with finite second moments on $\R^{d+2}$. Let 
\begin{enumerate}
\item $(\bar a, \bar w)$ be a solution to \eqref{eq ode version gradient flow} for fixed $\pi^0$ and
\item $\pi$ be a solution of the Wasserstein gradient flow
\[\showlabel\label{eq gradient flow pde}
\dot \pi = \div\big(\rho\,\nabla V\big), \quad V(a,w; \pi) = \int_{\R^d}\big(f_\pi - f^*\big)(x)\,a\sigma(w^Tx)\,\P(\d x)
\]
of \eqref{eq extended risk} with initial condition $\pi^0$. 
\end{enumerate}
Then $\pi(t) = (\bar a, \bar w)(t)_\sharp \pi^0$ for all $t$ and in particular $f_{\pi(t)} = f_{\pi^0; \bar a, \bar w}$.
\end{lemma}

The content of \eqref{lemma equivalent flows} is used in the proofs of \cite{chizat2018global} and \cite{relutraining}, but not stated explicitly. 

\begin{proof}
Let $(\bar a, \bar w)$ be a solution to \eqref{eq ode version gradient flow} and define
\[
\pi(t) = (\bar a, \bar w)(t)_\sharp \pi^0, \qquad X(a,w; \pi):= \nabla_{(a,w)} \int_{\R^d}\big(f_\pi - f^*\big)(x)\,a\sigma(w^Tx)\,\P(\d x).
\]
Note that $f_\pi = f_{\pi^0; \bar a, \bar w}$ by definition of the push-forward. By \cite[Proposition 4]{ambrosio2008transport}, we see that $\pi$ solves the continuity equation \eqref{eq gradient flow pde}, which coincides with the Wasserstein gradient flow of $\pi$ \cite[Appendix B]{chizat2018global}.
\end{proof}

\begin{remark}[Eulerian vs Lagrangian descriptions]
The different perspectives on the gradient flow training of infinitely wide two-layer neural networks have an analogue in classical fluid mechanics:
\begin{enumerate}
\item {\em Wasserstein gradient flow.} The parameter space $\R^{d+2}$ remains fixed, particles are referred to by their current position $(a,w)$. The distribution of particles in space $\pi$ evolves over time. This is an {\em Eulerian} perspective.
\item {\em $L^2$-gradient flow.} Particles $(\bar a, \bar w)$ are specified by their initial position $(\bar a, \bar w) = (a,w)$ and tracked over time. This is the {\em Lagrangian} perspective.
\end{enumerate}
\end{remark}

\begin{remark}
According to Lemma \ref{lemma equivalent flows}, the gradient flow training of infinitely wide neural networks can be studied in terms of $L^2$-gradient flows as well as Wasserstein gradient flows. The second perspective is conceptually simpler and has been used successfully to prove the major convergence results on gradient flow training in \cite{chizat2018global} and \cite{relutraining} (which were formulated in the second framework). The link between the two perspectives is through the method of characteristics for the continuity equation.

While the interpretation of gradient flow training as a Wasserstein gradient flow is appealing, we note that the link to optimal transport theory has not been exploited in depth. The crucial mathematical analysis so far has been conducted on the level of the ODE \eqref{eq ode version gradient flow} and interpreted through the PDE \eqref{eq gradient flow pde}.
\end{remark}

\begin{remark}
A similar approach has been pursued in \cite{nguyen2020rigorous} for multi-layer networks. 
\end{remark}

\subsection{Summary}
We briefly summarize the different ways to parametrize Barron functions which we described above. 

\vspace{2ex}

\begin{tabular}{lll}
\vspace{.5mm} {\bf Perspective} & {\bf Parametrizing object} &{\bf Optimization}\\
\hline
\vspace{1mm} Parameter distribution & probability distribution $\pi$ on $\R^{d+2}$ & Wasserstein gradient flow\\ 
Spherical graph & coefficient function $\hat a:S^{d}\to \R$ & see below\\
\vspace{1mm}  & first layer distribution $\hat \pi$ on $S^{d}$\\
Signed measure & signed Radon measure $\hat\mu$ on $S^{d}$ &unrelated to gradient flows\\
\vspace{1mm} & or $\mu$ $\R^{d+2}$\\ 
Indexed particles & Coefficient functions& $L^2$-gradient flow\\ 
	& $(a,w,b):(\Omega, \mathcal A, \bar\pi)\to \R^{d+2}$ 
\end{tabular}

\vspace{2ex}

The signed measure representation is crucial in our derivation of global and pointwise properties of Barron functions in Sections \ref{section special cases} and \ref{section structure}. It is, however, inconvenient from the perspective of gradient-flow based parameter optimization. A natural approach to optimization in this perspective via the Frank-Wolfe algorithm has been discussed in \cite{bach2017breaking}.

Gradient-flow based optimization of finite neural networks has direct analogues in the `parameter distribution' and `indexed particle' perspectives.

A natural optimization algorithm for spherical graphs is the $L^2$-gradient flow for the coefficient function $\tilde a$ which leaves the coefficients of the first layer frozen. This algorithm recovers random feature models rather than two-layer neural networks. Without proof we claim that if the first-layer distribution $\pi^0$ follows a Wasserstein gradient flow on the sphere, this corresponds to a gradient flow in which the first layer is norm-constrained. If the distribution $\pi^0$ were allowed to evolve on the entire parameter space $\R^{d+1}$ of $(w,b)$, then we conjecture we could recover a mixed perspective between indexed particles (second layer) and parameter distribution (first layer).

In the indexed particle perspective, we only considered the case that the index space was a $\R^{d+2}$ with a parameter distribution $\pi_0$ or the unit interval equipped with Lebesgue measure.

We conclude with a summary of Barron space from the perspective of function approximation. In all cases, the object in parameter space which represents a given function is non-unique and an infimum has to be taken in the definition of the norm. The probability integrals (i.e.\ the integrals in the first, second and last line) can be written as expectations.

\vspace{2ex}

\begin{tabular}{lll}
\vspace{.5mm} {\bf Perspective} & {\bf Representation} &{\bf Barron-norm}\\
\hline
\vspace{2mm} Parameter distribution & $\int_{\R^{d+2}}a\,\sigma(w^Tx)\,\pi(\d a\otimes \d w)$& $\int_{\R^{d+2}}|a|\,|w|\,\pi(\d a\otimes \d w)$\\ 
\vspace{2mm} Spherical graph & $\int_{S^d} \hat a(w)\,\sigma(w^Tx)\,\hat\pi(\d w)$& $\|\hat a\|_{L^p(\hat\pi)}, \:p\in [1,\infty]$\\
\vspace{2mm} Signed measure & $\int_{S^d} \sigma(w^Tx)\,\hat\mu(\d w)$ &$\|\hat\mu\|_{\M(S^d)}$\\
\vspace{2mm} & $\int_{\R^{d+1}} \sigma(w^Tx)\,\mu(\d w)$ &$\int_{\R^{d+1}}|w|\,\mu(\d w)$\\
Indexed particles & $\int_{\Omega} a_\theta \,\sigma(w_\theta^Tx)\,\bar\pi(\d\theta)$& $\int_{\Omega} |a_\theta|\,|w_\theta|\,\bar\pi(\d\theta)$
\end{tabular}

\section{Properties of Barron space}\label{section properties}

By Theorem \ref{theorem Barron Banach}, we know that $\B(\P)$ is a Banach space. In Section \ref{section special cases}, we will characterize $\B(\P)$ up to isometry in two special cases and conclude that generally, $\B(\P)$ is neither reflexive nor separable. Here we will discuss the relationship of Barron space with classical function spaces on the one hand and finite two-layer networks on the other. Most results in this section are known in other places; some are reproved to illustrate the power of different parametrizations.

\subsection{Relationship with other function spaces}
We briefly explore the relationship of Barron space and more classical function spaces. We begin by the relationship to the Barron class $X$ discussed in the introduction. Denote by $\hat f$ the Fourier transform of a function $f$ and
\[
\|f\|_X = \int_{\R^d}\big|\hat f(\xi)\big|\,\big[1+|\xi|^{\sw2}\big]\,\d\xi, \qquad X = \big\{f\in L^1_{loc}(\R^d)\::\:\|f\|_X<\infty\big\}.
\]
Clearly, $X$ is a Banach space. \sw{The exponent in the weight $|\xi|^2$ cannot be lowered as shown in \cite[Proposition 7.4]{barron_boundaries}.} We recall a classical result in modern terms, which can be found in \cite[Section IX, point 15]{barron1993universal} and \sw{\cite[Theorem 2]{klusowski2018approximation}.}

\begin{theorem}\label{theorem smooth functions are barron}
\begin{enumerate}
\item Let $s> \frac d2+ \sw{2}$. Then $H^s(\R^d)$ embeds continuously into $X$.
\item Assume that $\spt(\P)$ is bounded. Then $X$ embeds continuously in $\B(\P)$ with constant $4\,\sup_{x\in\spt(\P)} |x|$.
\end{enumerate}
\end{theorem}

The second statement is only implicit in the proof of \cite[Proposition 1]{barron1993universal}, which proceeds by showing that
\[
X\subseteq \overline{\operatorname{conv}(\F_\infty)} = \B(\P).
\] 
So every sufficiently smooth function is Barron, see also \cite[Appendix B.3]{relutraining} for the proof in the context of fractional Sobolev spaces. 

\begin{remark}
The smoothness required to show that a function is Barron increases with dimension. Depending on the purpose, the space $H^s(\R^d)$ can be fairly large when $s>\frac d2+2$. For example, there are $H^s$-functions for which Sard's theorem fails. Thus we conclude that Sard's theorem does not hold in \sw{Barron space.}
\end{remark}

On the other hand, every Barron function is at least Lipschitz continuous.

\begin{theorem}\label{theorem barron functions are lipschitz}
Assume that $\P$ is a Borel probability measure with finite first moment. 
Denote by $C^{0,1}(\P)$ the space of (possibly unbounded) Lipschitz functions with the norm
\[
\|f\|_{C^{0,1}(\P)} = \|f\|_{L^1(\P)} + \sup_{x\neq y} \frac{|f(x) - f(y)|}{|x-y|}. 
\]
Then $\B(\P)$ embeds continuously into $C^{0,1}(\P)$.
\end{theorem}

\begin{proof}
We represent $f=f_\mu$ by a signed Radon measure on $S^d$. Then
\begin{align*}
|f(x) - f(y)| &= \left|\int_{S^d} \sigma(w^Tx) - \sigma(w^Ty)\,\mu(\d w) \right|\\
	&\leq \int_{S^d} |w^T(x-y)|\,|\mu|(\d w)\\
	&\leq |x-y|\,\|\mu\|_\M.
\end{align*}
By taking the infimum over $\mu$, we find that
\[
\sup_{x\neq y} \frac{|f(x) - f(y)|}{|x-y|} \leq \|f\|_{\B(\P)}
\]
for all $f\in \B(\P)$. Note that $f_\mu$ is defined on the whole space, so
\begin{align*}
|f_\mu(0)| &= \int_{S^d} \sigma(w_{d+1})\,\mu(\d w) \leq \|f\|_{\B(\P)}
\end{align*}
is well-defined even if $0\notin\spt\,\P$. Thus
\begin{align*}
\|f\|_{L^1(\P)} &\leq \int_{\R^d} |f_\mu(0)| + |f_\mu(x) - f_\mu(0)|\,\P(\d x)\\
	&\leq \|f\|_{\B(\P)} \left[1 + \int_{\R^d}|x|\,\P(\d x)\right].
\end{align*}
\end{proof}

\begin{remark}
If $\P$ has bounded support, $C^{0,1}(\P) = C^{0,1}(\spt \,\P)$ with equivalent norms, where
\[
\|f\|_{C^{0,1}(K)} = \|f\|_{L^\infty(K)} + \sup_{x\neq y} \frac{|f(x) - f(y)|}{|x-y|}
\]
for any compact set $K\subset \R^d$. Even if $\P$ has unbounded support, we can consider an equivalent norm
\[
\|f\|_{C^{0,1}}' = |f(a)| + \sup_{x\neq y} \frac{|f(x) - f(y)|}{|x-y|}
\]
for any fixed $a\in \spt(\P)$.
\end{remark}

\begin{remark}
Barron space also embeds into the compositional (or `flow-induced') function class for infinitely deep ResNets \cite[Theorem 9]{weinan2019lei}. \sw{In the statement of \cite[Theorem 9]{weinan2019lei}, a suboptimal inequality of the form $\|f\|_{comp} \leq 2\|f\|_\B +1$ is proved. Note that this can be improved to $\|f\|_{comp}\leq 2\|f\|_\B$ by a simple scaling argument.}
\end{remark}

\begin{remark}\label{remark complexity}
$\B(\P)$ has favorable properties in the context of statistical learning theory. Namely, the unit ball in $\B(\P)$ has low Rademacher complexity \cite[Theorem 6]{weinan2019lei} and thus low generalization error \cite[Theorem 4.1]{E:2018ab}.
\end{remark}

\subsection{Relationship to two-layer networks}
In Section \ref{section representations}, we derived eight different representations for general Barron functions. While Barron space is a natural model for infinitely wide two-layer networks, it is not the only possible choice and parameter initialization is key in determining the correct limiting structure. 

The generalization bounds mentioned in Remark \ref{remark complexity} are one reason why the path-norm on a neural network is considered in the first place; another is that it is easy to bound in terms of the network weights. The direct and inverse approximation theorems (Theorems \ref{inverse approximation theorem} and \ref{direct approximation theorem} below) establish that the correct space to consider the infinite neuron limit in under this norm is Barron space. Theorem \ref{theorem dynamics} shows that the space is stable under gradient-flow dynamics.

\begin{theorem}[Compactness and Inverse Approximation]\label{inverse approximation theorem}
Let $f_m(x) = \sum_{i=1}^{N_m} a_i^m\,\sigma\big((w_i^m)^Tx\big)$ for some $N_m<\infty$ and assume that $\sum_{i=1}^{N_m} |a_i|\,|w_i| \leq 1$ for all $m\in\N$. 
\begin{enumerate}
\item If $\P$ has finite $p$-th moments, then there exists a subsequence $m_k\to\infty$ and $f\in \B(\P)$ such that $f_{m_k}\to f$ strongly in $L^q(\P)$ for all $q<p$.
\item If $\P$ has compact support, then the convergence even holds in $C^{0,\alpha}(\spt\,\P)$ for all $\alpha<1$.  
\end{enumerate}
\end{theorem}

Thus Barron space includes the limiting objects of norm-bounded finite neural networks, i.e.\ Barron space is large enough to include all relevant models for infinitely wide neural networks (with finite path-norm). The theorem was originally proved in \cite[Theorem 5]{weinan2019lei}. We reprove it here using the signed Radon measure representation.

\begin{proof}
{\bf General set-up.} Without loss of generality, we assume that $w_i^m\neq 0$ for all $i, m$. We can write $f_{m}(x) = \int_{S^d} \sigma(w^Tx)\,\mu_{m}(\d w)$ where $\mu_{m} = \sum_{i=1}^{N_{m}}a_i^m|w_i^m| \delta_{w_i^m/|w_i^m|}$. The norm bound corresponds to the estimate $\|\mu_m\|_\M \leq 1$, so by the compactness theorem for Radon measures \cite[Section 1.9]{evans2015measure}, there exist a subsequence of $\mu_m$ (relabelled) and a finite Radon measure $\mu$ on $S^d$ such that $\mu_m \stackrel *\wto \mu$, i.e.\
\[
\int g(w)\,\mu_m(\d w) \to \int g(w)\,\mu(\d w)\qquad\forall g \in C^0(S^d).
\]
In particular, $f_{\mu_m}\to f_\mu$ pointwise almost everywhere. 

{\bf Compact support.} If $\spt(\P)$ is bounded, then $f_m = f_{\mu_m}$ is bounded in $C^{0,1}(\spt\,\P)$, since $\B(\P) \embeds C^{0,1}(\spt\,\P)$ continuously. Since $C^{0,1}(\spt\,\P)\embeds C^{0,\alpha}(\spt\,\P)$ compactly for all $\alpha<1$, we find that there exists $f^*\in C^{0,1}(\spt\,\P)$ such that $f_m\to f^*$ strongly in $C^{0,\alpha}(\spt\,\P)$. Clearly $f^* = f_\mu$.

{\bf Bounded moments.} If the $p$-th moment of $\P$ is bounded, then $f_m$ is uniformly bounded in $L^p(\P)$ due to the bound on $f_{\mu_m}(0)$ and the uniform Lipschitz condition. For all $R>0$, we find that
\begin{align*}
\int_{\R^d} |f_{m} - f_\mu|^q\,\P(\d x) &\leq \int_{\R^d} \min\{|f_m - f_\mu|, R\}^q \,\P(\d x) + \int_{\R^d}\max\{|f_m-f_\mu|-R, 0\}^q\,\P(\d x) 
\\
	&\leq \int_{\R^d} \min\{|f_m - f_\mu|, R\}^q \,\P(\d x) + 2\,R^{q-p} \int_{\R^d} \big[1 + |x|\big]^p\,\P(\d x)
\end{align*}
by Chebyshev's inequality. The first term on the right converges to $0$ as $m\to \infty$ due to the dominated convergence theorem while the second term converges to zero when we take $R\to \infty$ in a second step.
\end{proof}

Conversely, any function in $\B(\P)$ can be approximated by finite networks in $L^2(\P)$. 

\begin{theorem}[Direct Approximation]\label{direct approximation theorem}
Assume that $\P$ has finite second moments and let $f\in \B(\P)$. Then for any $m\in\N$ there exist $a_i, w_i$ such that
\[
\left\|f - \sum_{i=1}^m a_i\,\sigma(w_i^T\cdot)\right\|_{L^2(\P)} \leq\frac{2\,\|f\|_{\B(\P)}\sw{\left(\int_{\R^d}|x|^2+1\,\P(\d x)\right)^{1/2}} }{\sqrt m}, \qquad \sum_{i=1}^m |a_i|\,|w_i| \leq 2 \,\|f\|_{\B(\P)}.
\]
\end{theorem}

In this sense, Barron space is the smallest Banach space which contains all finite neural networks. The result can be deduced from \sw{the Maurey-Barron-Jones Lemma in} Hilbert space geometry \cite[Theorem 1 and Lemma 1]{barron1993universal} or Monte-Carlo integration \cite[Theorem 4]{weinan2019lei}. Both proofs use probabilistic arguments and the normalized representation
\[
f_m(x) = \frac1m \sum_{i=1}^m a_i\,\sigma(w_i^Tx), \qquad f(x) = \int_{\R^{d+2}} a\,\sigma(w^Tx)\,\pi(\d w)
\] 
where the variables $(a_i, w_i)$ are drawn iid from the distribution $\pi$. \sw{If the support of the data distribution $\P$ is compact, the direct approximation theorem can be improved from $L^2$-approximation to $L^\infty$-approximation, and the rate $m^{-1/2}$ can be improved to $\sqrt{\log(m)}\,m^{-1/2-1/d}$ \cite[Theorem 1]{klusowski2018approximation}.
}

\sw{
\begin{remark}
Concerning approximation by a {\em finite} set, the metric entropy of the unit ball of Barron space $\B(\P)$ in $L^2(\P)$ has been calculated in \cite[Theorem 1]{siegel2021optimal} in the case where $\P$ is the uniform distribution on the unit ball in $\R^d$. In particular, the authors find that for $n\in \N$ there exist $f_1,\dots, f_n$ Barron functions such that
\[
\|f\|_{\B(\P)} \leq 1 \qquad \Ra\qquad \exists\ 1\leq i\leq n\quad\text{s.t. } \|f-f_i\|_{L^2(\P)} \leq C\,m^{-\frac12 - \frac 3d}
\]
for some constant $C$ which may depend on $d$ but not $n$. The exponent $-\frac12 - \frac3d$ is sharp.
\end{remark}
}

\begin{remark}\label{remark countable network}
We can represent finite neural networks equivalently as $f(x) = \sum_{i=1}^m a_i\,\sigma(w_i^Tx)$ or $f(x) = \frac1m \sum_{i=1}^m a_i'\,\sigma(w_i^Tx)$ with $a_i' = m\,a_i$.  

The second expression already resembles a Riemann sum or Monte-Carlo integral, so we easily passed to the infinite neuron limit in the parameter distribution representation in Section \ref{section parameter distribution}.

The first expression on the other hand is an unnormalized sum, and at first glance it appears that the limit should be a countable series as is the case for spaces of polynomials, Fourier series or eigenfunction expansions. Clearly this heuristic fails, since both representations induce the same continuum limit given their natural path norms. The difference to the usual setting is that the sum is not an expansion in a fixed basis. A better way to pass to the limit in the unnormalized sum representation is given in the proof of Theorem \ref{inverse approximation theorem}. 

The class of countably wide two-layer networks
\[
\widehat \F_\infty = \left\{\sum_{i=1}^\infty a_i \,\sigma(w_i^Tx)\:\bigg|\: \sum_{i=1}^\infty|a_i|\,|w_i|<\infty\right\}
\] 
is a closed subspace of $\B(\P)$ since the strong limit of a sequence of measures $\mu_n$, each of which is given by a sum of countably many atoms, is also a sum of countably many atomic measures. 
The closure of the unit ball of $\widehat \F_\infty$ in the $L^2(\P)$-topology is the unit ball of $\B(\P)$ by the inverse and direct approximation theorems, which is one reason why we prefer $\B(\P)$ over $\widehat\F_\infty$.

Another reason is this: Each function $\sigma(w^T\cdot) = \max\{w^Tx + w_{d+1}, 0\}$ fails to be differentiable along the space $\{(x,1): w^Tx+ w_{d+1} = 0\}$. It is easy to see that a function in $\widehat \F_\infty$ is either affine linear (if the singular sets coincide and the discontinuities cancel out) or fails to be $C^1$-smooth. In particular, the Barron criterion fails for all countably wide two-layer networks: $\widehat \F_\infty\cap X=\{0\}$.
\end{remark}

For those reasons we consider Barron space the correct function space for two-layer neural networks with controlled path norms. In addition, Barron space is stable under gradient flow training in the following sense. Recall that in the `mean field' regime, parameter optimization for two-layer networks is described by Wasserstein gradient flows, see \cite[Proposition B.1]{chizat2018global} or \cite[Appendix A]{relutraining}.

\begin{theorem}\cite[Lemma 3.3]{relutraining}\label{theorem dynamics}
Let $\pi^0$ be a parameter distribution on $\R^{d+2}$ such that 
\[
N_0:= \int_{\R^{d+2}} a^2 + |w|_{\ell^2}^2\,\pi^0(\d a\otimes \d w) < \infty.
\]
Assume that $\pi_t$ evolves by the 2-Wasserstein gradient flow of a risk functional
\[
\Risk(\pi) = \int_{\R^d\times \R} \ell\big(f_\pi(x), y\big)\,\P(\d x \otimes \d y)
\]
where $\ell$ is a sufficiently smooth convex loss function. Then there exists $\bar c>0$ such that
\[
\bar c\,\|f_{\pi_t}\|_{\B(\P)} \leq \int_{\R^{d+2}} a^2 + |w|^2\,\pi^0(\d a\otimes \d w) \leq 2\,\big[N_0 + \Risk(\pi^0)\,t\big]\qquad\forall\ t>0
\]
and $\limsup_{t\to\infty} t^{-1}\|f_{\pi_t}\|_{\B(\P)} = 0$.
\end{theorem}

The constant $\bar c$ depends on the equivalence constant between the Euclidean norm on $\R^d$ and the norm on $w$, which is chosen to be dual to the norm which we consider on data space $\R^d$ (i.e.\ in the $x$-variables). Consequences of this result are explored in \cite{dynamic_cod}.

\begin{remark}
There are other function spaces for two-layer networks. In a highly overparametrized scaling regime and given a suitable initialization, the gradient descent dynamics of network parameters are determined by an infinitely wide random feature model, the `neural tangent kernel'. Even at initialization, parameters are usually chosen such that the path norm of a two-layer network with $m$ neurons scales like $\sqrt{m}$. Instead, $a, w$ are chosen randomly in such a way that
\[
\mathbb E \left[\sum_i |a_i|^2|w_i|^2\right] = 2
\]
(He initialization).
\end{remark}

\section{Special cases}\label{section special cases}
\subsection{One-dimensional Barron functions}
Recall that $\B(\P)$ does not depend on $\P$, but only the collection of $\P$-null sets. Since Barron functions are (Lipschitz-)continuous, we observe that $f_\mu = f_\nu$ $\P$-almost everywhere if and only if $f_\mu\equiv f_\nu$ on $\spt(\P)$. For a closed set $A$ we therefore denote $\B(A) = \B(\P)$ for any $\P$ such that $\spt(\P) = A$.

The first example was originally given in \cite[Remark A.5]{approximationarticle}. It shows in a very broad sense that one-dimensional Barron space is the largest space with a weak second-order structure. 

\begin{example}\label{example 1d}
$\B[0,1]$ is the space of functions whose first derivative is in $BV(0,1)$/whose distributional derivative is a Radon measure on $[0,1]$. The norm
\[
\|f\|' = |f(0)| + |f'(0)| + \|f''\|_{\M[0,1]}
\]
is equivalent to $\|\cdot\|_{\B[0,1]}$. 
\end{example}

\begin{proof}
{\bf Inclusion in Barron space.} Assume for the moment that $f\in C^2[0,1]$. Then
\begin{align*}
f(x) &= f(0) + \int_0^x f'(\xi)\,\d\xi\\
	&= f(0) + \int_0^x f'(0) + \int_0^\xi f''(s)\ds\,\d\xi\\
	&= f(0) + f'(0)\,x + \int_0^x (x-\xi)\,f''(\xi)\,\d\xi\\
	&= f(0)\,\sigma(1) + f'(0)\,\sigma(x) + \int_0^1f''(\xi)\,\sigma(x-\xi)\,\d\xi
\end{align*}
and thus $\|f\|_{\B[0,1]} \leq \|f\|'$. The result holds by approximation also if the second derivative is merely a measure. 

{\bf Opposite inclusion.} If $f\in \B[0,1]$, there exists a Radon measure $\mu$ on $S^1$ such that
\begin{align*}
f(x) &= \int_{S^1} \sigma(wx + b) \,\mu(\d w \otimes \d b)\\
	&= \mu(\{0,1)\})\,\sigma(1) + \int_{\{w>0\}} \sigma\left(\frac{w}{|w|}x+\frac{b}{|w|}\right)\,|w|\,\mu(\d w\otimes \d b)+ \int_{\{w<0\}} \sigma\left(\frac{w}{|w|}x+\frac{b}{|w|}\right)\,|w|\,\mu(\d w\otimes \d b)\\
	&= \mu(\{0,1)\})\,\sigma(1) + \int_\R \sigma(x+\tilde b)\,\tilde\mu^1(\d \tilde b) + \int_\R \sigma(-x+\tilde b)\,\mu^2(\d \tilde b)
\end{align*}
where $\mu^{1,2} = T_\sharp\mu$ is the push-forward of the measure $|w|\cdot\mu$ on the domain $w>0$ or $w<0$ respectively along the map 
\[
T:\R^2\setminus \to \R,\quad (w,b) \mapsto \frac{b}{|w|}.
\]
We note that or $x\in[0,1]$ we have $\sigma(x+\tilde b) = 0$ if $\tilde b< -1$ and 
\begin{align*}
\int_\R \sigma(x+\tilde b)\,\tilde\mu^1(\d \tilde b) &= \int_{-1}^0 \sigma(x+\tilde b) \,\tilde\mu^1(\d \tilde b) + \int_{0}^\infty \sigma(x+\tilde b) \,\tilde\mu^1(\d \tilde b)\\
	&= \int_{-1}^0 \sigma(x+\tilde b) \,\tilde\mu^1(\d \tilde b) + \left(\int_0^\infty 1\,\tilde\mu(\d\tilde b)\right)x + \left(\int_0^\infty \tilde b\,\tilde\mu(\d\tilde b)\right)\\
\int_\R \sigma(-x+\tilde b)\,\mu^2(\d \tilde b) &= \int_0^1 \sigma(-x+\tilde b)\,\mu^2(\d \tilde b) +\left( \int_{-\infty}^0 1\,\tilde\mu^2(\d \tilde b) \right)x + \left( \int_{-\infty}^0 \tilde b\,\tilde\mu^2(\d \tilde b) \right).
\end{align*}
We can ignore the linear terms when computing the (distributional) second derivative. We claim that $f_\mu'' = \tilde \mu^1 + (-\id)_\sharp \tilde\mu^2$. This is easily verified formally by noting that $\sigma'' = \delta$, i.e.\ $\sigma$ is a Green's function for the Laplacian in one dimension.

More formally, take $g\in C_c^\infty(0,1)$ and use Fubini's theorem and integration by parts to obtain
\begin{align*}
\int_0^1 \left(\int_0^1 \sigma(-x+\tilde b)\,\mu^2(\d \tilde b)\right)\,g''(x) \dx &= \int_0^1 \left(\int_0^{\tilde b}(\tilde b -x)\,g''(x)\,\dx\right)\,\tilde\mu^2(\d\tilde b)\\
	&= \int_0^1 \left(0-\int_0^{\tilde b} (-1)g'(x)\dx \right)\,\tilde\mu^2(\d\tilde b)\\
	&= \int_0^1 g(\tilde b) \,\mu^2(\d\tilde b).
\end{align*}
The boundary terms vanish since $g(0) = g'(0) = 0$ and $\tilde b - \tilde b=0$. The same argument can be applied to the second term. We have shown more generally that $|f(0)| + |f'(0)| \leq C\,\|f\|_{\B[0,1]}$ and finally observe that
\[
\|f''\|_{\M[0,1]} \leq \|\tilde \mu^2\|_{\M[0,1]} + \|\tilde \mu^1\|_{[-1,0]} \leq \int_0^1 \sqrt{1+ \tilde b^2}\,\big|\tilde\mu^2\big|(\d\tilde b) + \int_{-1}^0 \sqrt{1+ \tilde b^2}\,\big|\tilde\mu^1\big|(\d\tilde b) \leq\|\mu\|.
\]
Taking the infimum over $\mu$, we find that also $\|f''\|_{\M[0,1]}\leq \|f\|_{\B[0,1]}$.
\end{proof}

In particular, since the space of Radon measures is neither separable nor reflexive, we deduce that generally $\B(\P)$ is neither separable nor reflexive.

\begin{remark}\label{remark equivalent norm 1d}
The same argument shows that $\B(\R)$ is isomorphic to the space of functions whose second derivatives are finite Radon measures with finite first moments and the norm
\[
\|f\|_{\B(\R)}' = |f(0)| + |f'(0)| + \int_\R \sqrt{1+ b^2}\,|f''|(\d b).
\]
In particular, non-constant periodic functions are never in $\B(\R)$.
\end{remark}

\begin{remark}
$BV(0,1)$ embeds into $L^\infty(0,1)$. Barron space is a proper subspace of the space of Lipschitz functions (function whose first derivative lies in $L^\infty$) and -- heuristically -- we can imagine it to be roughly as large in the space of Lipschitz functions as $BV$ is in $L^\infty$.
\end{remark}

\begin{remark}
$\B[0,1]$ is an algebra (i.e.\ if $f, g\in \B[0,1]$ then also $fg\in\B[0,1]$). This is generally not true, see Remark \ref{remark not an algebra}.
\end{remark}

\subsection{Positively one-homogeneous Barron functions}
It has been recognized since \cite{bach2017breaking} that the space of positively homogeneous Barron functions is significantly easier to understand than full Barron space.

 Denote by $\R\P^{d-1}$ the $d-1$-dimensional real projective space, i.e.\ the space of undirected lines in $\R^d$, which we represent as the quotient $\R\P^{d-1}= S^{d-1}/\sim$ of the unit sphere under the equivalence relation which identifies $w$ and $-w$. Without loss of generality, we assume that $x, w$ are normalized with respect to the Euclidean norm on $\R^{d}$.

\begin{lemma}
Set $\B_{hom}(\R^d)= \{f\in \B(\P) : f(rx) = r\,f(x)\:\forall\ r>0\}$. Then $\B_{hom}(\P)$ is isomorphic to the product space $\M(\R\P^{d-1})\times \R^{d}$ where $\M(\R\P^{d-1})$ denotes the space of Radon measures on $d-1$ dimensional real projective space.
\end{lemma}

\begin{proof}
{\bf Dimension reduction.} Let $f = f_\mu \in \B(\R^d)$ be a positively one-homogeneous function. Then for any $\lambda>0$, the identity
\begin{align*}
f(x) &= \frac{f(\lambda x)}{\lambda}\\
	&= \int_{S^d} \frac{\sigma\big(w^T(\lambda x)+b\big)}\lambda\,\mu(\d w\otimes \d b)\\
	&= \int_{S^d} \sigma\left(w^Tx + \frac b\lambda\right)\,\mu(\d w\otimes \d b)
\end{align*}
holds. We can pass to the limit $\lambda\to\infty$ by the dominated convergence theorem and obtain
\begin{align*}
f(x) &= \int_{S^d}\sigma(w^Tx)\,\mu(\d w\otimes \d b)\\
	&= \int_{S^d} \sigma\left(\frac{w}{|w|}^Tx\right)\,|w|\,\mu(\d w\otimes \d b)\\
	&= \int_{S^{d-1}} \sigma(w^Tx)\,\hat \mu(\d w)
\end{align*}
where $\hat \mu$ is the push-forward of $|w|\cdot \mu$ along the map $(w,b)\mapsto w$. Clearly $\|\hat\mu\|\leq \|\mu\|$, so without loss of generality we may assume that $f\in \B_{hom}$ is represented by a measure $\mu$ on $S^{d-1}$.

{\bf Odd-even decomposition.}
If $\mu$ is a signed Radon measure on $S^d$, we decompose $\mu = \mu^{even} + \mu^{odd}$ where 
\[
\mu^{even} = \frac{\mu + T_\sharp \mu}2, \qquad \mu^{odd} = \frac{\mu - T_\sharp \mu}2, \qquad T:S^{d-1}\to S^{d-1}, \quad T(x) = -x.
\]
Note that
\[
\|\mu^{even/odd}\| \leq \frac{\|\mu\| + \|T_\sharp\mu\|}2 = \|\mu\|, \qquad \|\mu\|= \big\|\mu^{even} + \mu^{odd}\big\| \leq \|\mu^{even} \| + \|\mu^{odd}\|.
\]
In particular,
\[
\|f\|_{\B_{hom}} = \inf_{\{\mu: f_\mu = f\}}\|\mu^{odd}\|_{\M(S^{d-1})} + \|\mu^{even}\|_{\M(S^{d-1})}
\]
is equivalent to the norm on $\B_{hom}$ induced by the norm of $\B(\R^d)$. We further find that
\begin{align*}
f_\mu(x) &= \int_{S^d} \sigma(w^Tx)\,\mu(\d w)\\
	&= \int_{S^d} \frac{\sigma(w^Tx) + \sigma(-w^Tx)}2 + \frac{\sigma(w^Tx) - \sigma(-w^Tx)}2\,\mu(\d w)\\
	&= \frac12 \int_{S^d}|w^Tx|\,(\mu^{even} + \mu^{odd})(\d w) + \frac12\left(\int_{S^d}w^T\,(\mu^{even} + \mu^{odd})(\d w)\right)^Tx\\
	&= \frac12 \int_{S^d}|w^Tx|\,\mu^{even}(\d w) + \frac12\left(\int_{S^d}w^T\,\mu^{odd}(\d w)\right)^Tx
\end{align*}
since the other integrals drop out by symmetry. In particular, $f_\mu$ naturally decomposes into an even part $f_\mu^{even} = f_{\mu^{even}}$, and a linear (in particular odd) part $f^{lin}$. Clearly
\[
f_\mu \equiv f_{\mu'} \qquad\LRa\qquad f_{\mu}^{even} = f^{even}_{\mu'}, \quad f_{\mu}^{lin} = f_{\mu'}^{lin}.
\]

{\bf Linear part.}
The linear function $f(x) = \alpha^Tx$ is $\|\alpha\|$-Lipschitz, so $\|f\|_{\B} \geq \|\alpha\|$. On the other hand,
\[
f = f_\mu \qquad\text{where}\qquad \mu = \|\alpha\|\,\big[\delta_{\alpha/\|\alpha\|} - \delta_{-\alpha/\|\alpha\|}\big]\qquad\Ra\qquad \|f\|\leq 2\,\|\alpha\|.
\]
Taking the infimum over all odd measures shows that
\[
\|\alpha\| \leq \|\mu^{odd}\| \leq 2\,\|\alpha\|.
\]

{\bf Even part.} We can interpret $\mu^{even}$ as a signed Radon measure on $\R\P^{d-1}$. To conclude the proof, it suffices to show that the map $\mu^{even}\mapsto f_{\mu^{even}}$ is injective. Assume that $f_{\mu^{even}} = 0$, i.e.
\[
\int_{S^{d-1}} |w^T x|\,\mu^{even}(\d w) = 0\qquad \forall\ x\in S^{d-1}.
\]
We first consider the case $d=2$ and identify $S^{d-1} = (0,2\pi)$ via the usual map $\phi\mapsto (\cos\phi,\sin\phi)$. Write
\[
x= (\cos\theta, \sin\theta), \qquad \langle x, w\rangle = \cos\theta\,\cos\phi + \sin\theta \sin \phi = \cos(\phi-\theta).
\]
{\em Claim:} The space generated by the family $\{\cos(\cdot - \theta)\}_{\theta\in [0,2\pi)}$ is $C^0$-dense in the space of continuous $\pi$-periodic functions on $\R$. {\em Proof of claim:} Note that $f(\phi) = |\cos\phi|$ satisfies $f'' + f = -\sum_{k\in \Z} \delta_{k\pi}$, so if $g$ is a $\pi$-periodic $C^2$-function, then for any $\theta\in [0,2\pi)$
\[
\lim_{h\to 0} \int_0^{2\pi} \left[\frac{|\cos|(\phi + h-\theta) -2\,|\cos|(\phi-\theta) + |\cos| (\phi - h -\theta)}{h^2} + |\cos|(\phi-\theta)\right]g(\phi)\,\d\phi = -2\,g(\theta).
\]
For any $\eps>0$, we can choose $h$ sufficiently small and approximate the integral by a Riemann sum with $N$ terms in such a way that
\[
\sup_\theta \left|\sum_{i=1}^N\frac{g(\phi_i)}2\,\left[\frac{|\cos|(\phi_i + h-\theta) -2\,|\cos|(\phi_i-\theta) + |\cos| (\phi_i - h -\theta)}{h^2} + |\cos|(\phi_i-\theta)\right] - g(\theta)\right|<\eps.
\]
Up to rearranging the sum and passing back to the original coordinates, this proves the claim since $C^2$ is dense in $C^0$. \qedsymbol

Now let $d\geq 2$. By the two-dimensional result, every function of the form $f(x) = g(|\langle x, v\rangle|)$ can be approximated arbitrarily well in any fixed plane spanned by $v$ and any $\tilde v\bot v$ in a way which is constant in directions orthogonal to the plane. By averaging the approximating functions $f_{\tilde v}$ over the choice of plane, we find that $f$ can be approximated arbitrarily well by functions of the form $|\langle w, x\rangle|$ on the whole sphere. Cancellations do not occur since the function (asymptotically) only depends on one the direction which all planes share. Again, we replace the averaging integral by a Riemann sum.

By a weaker version of the Universal Approximation Theorem \cite{cybenko1989approximation}, sums of functions depending only on a single direction (ridge functions) are dense in $C^0(S^{d-1})$.
\end{proof}

\begin{remark}
We observe that the kernel of the map 
\[
\M(S^{d-1}) \to C^{0,1}(S^{d-1}), \qquad \mu\mapsto f_\mu
\]
is the subspace
\[
N= \left\{\mu \in \M(S^{d-1}) \:\bigg|\: (-\id)_\sharp\mu = -\mu, \: \int_{S^{d-1}}w\,\mu(\d w) = 0\right\}.
\]
\end{remark}

\begin{remark}
Since the space of Radon measures is not separable or reflexive for $d\geq 2$, neither is $\B_{hom}$.
The space $\B_{hom}$ is a closed subspace of $\B(\R^d)$, so $\B(\R^d)$ is neither separable nor reflexive. 
\end{remark}

\begin{example}\label{example homogeneous}
Some functions in $\B_{hom}$ are 
\begin{enumerate}
\item $\sigma(w^Tx)$ for any $w\in \R^d$.
\item the Euclidean norm $f(x) = \|x\|_{\ell^2}$. Up to constant, we can write $f$ as an average over $\sigma(w^Tx)$ over the uniform distribution $\pi^0$ on the unit sphere
\[
f(x) = c_d \int_{S^{d-1}} \sigma(w^Tx)\,\pi^0(\d w), \qquad c_d = \left[\int_{S^{d-1}}\sigma(w_1)\,\pi^0(\d w)\right]^{-1}  \sim 2\,\sqrt{\pi d}.
\]
since
\begin{align*}
\int_{S^{d-1}}\sigma(w_1)\,\pi^0(\d w) &= \frac1{|S^{d-1}|} \int_0^1w_1\,|S^{d-2}|\,\big(1-w_1^2\big)^\frac{d-2}2\,\d w_1\\
	&= \frac{|S^{d-2}|}{|S^{d-1}|}\int_0^1 -\frac1d\,\frac{d}{dw} (1-w^2)^\frac{d}2\,\d w\\
	&= \frac{|S^{d-2}|}{d\,|S^{d-1}|}.
\end{align*}
This can be computed explicitly as
\begin{align*}
\frac{|S^{d-2}|}{d\,|S^{d-1}|}
	&= \frac{ \frac{2 \,\pi^\frac{d-1}2}{\Gamma(\frac{d-1}2)}}{d\,\frac{2\,\pi^\frac d2}{\Gamma(\frac d2)}}
	= \pi^{-1/2}\,\frac{\Gamma(d/2)}{d\,\Gamma\big((d-1)/2\big)}
	= \frac{d-1}{2\,d}\pi^{-1/2}\,\frac{\Gamma(d/2)}{\frac{d-1}2\,\Gamma\big((d-1)/2\big)}\\
	&= \frac{d-1}{2\,d}\pi^{-1/2}\,\frac{\Gamma(d/2)}{\Gamma\big((d+1)/2\big)}
	\sim \frac{1}{2\,\sqrt{\pi}}\,\frac{\sqrt{\frac{2\pi}{d/2}}\left(\frac{d}{2e}\right)^\frac{d}{2}} {\sqrt{\frac{2\pi}{(d+1)/2}}\, \left(\frac{d+1}{2e}\right)^\frac{d+1}2} \sim \frac{1}{2\,\sqrt{\pi}} \sqrt{\frac{2e}{d+1}}
	\sim \sqrt{\frac{e}{2\pi\,d}}.
\end{align*}
by Stirling's formula.

\item In the first example, $f$ was non-differentiable along a hyperplane, while in the second example, $f$ was smooth except at the origin (where any non-linear positively one-homogeneous function is non-differentiable). We can use the same argument as in the second example to express $f(x) = \sqrt{x_1^2 + \dots + x_k^2}$ for any $k\leq d$, which is singular along a single $d-k$-dimensional subspace.

\item Countable sums of these examples (or rotation thereof) with $\ell^1$-weights lie in $\B_{hom}$.
\end{enumerate}
\end{example}

\section{Structure of Barron functions}\label{section structure}

\subsection{Limits at infinity} Barron functions grow at most linearly at infinity. We show that they are well behaved in a more precise sense.

\begin{theorem}
For any $f\in \B(\R^d)$, the function
\[
f_\infty:S^{d-1}\to \R, \qquad f_\infty(x) = \lim_{r\to\infty} \frac{f(rx)}{r}
\]
is well-defined and a Barron function on $S^d$.
\end{theorem}

\begin{proof}
Write $f= f_\mu$ for a suitable signed Radon measure $\mu$ on $S^d$. In this argument, we write $(w,b)$ instead of $w$ and $(x, 1)$ instead of $x$ like above since the last entry does not scale. Note that
\begin{align*}
\frac{f(rx)}{r} &= \frac1r \int_{S^d} \sigma\big(w^T(rx) + b\big)\,\mu(\d w)\\
	&= \int_{S^d} \sigma\left(w^Tx + \frac{b}r\right)\,\mu(\d w)\\
	&\to \int_{S^d} \sigma(w^Tx)\,\mu(\d w)
\end{align*}
as $r\to\infty$ by the dominated convergence theorem. The result is immediate. 
\end{proof}

The function $f_\infty$ captures the linearly growing component of $f$ at infinity. Note that Barron functions like $f(x) = \sigma(x_1+1) - \sigma(x_1)$ may also be bounded. We prove that there is nothing `in between' the bounded and the linearly growing regime.

\begin{theorem}
Let $f\in \B(\R^d)$ such that $f_\infty\equiv 0$. Then $f$ is bounded.
\end{theorem}

\begin{proof}
{\bf Reduction to one dimension.} Assume that $f_\infty\equiv 0$. Then for every $\nu \in S^{d-1}$, the one-dimensional Barron function $g(r) = f(r\nu)$ satisfies $\lim_{r\to\infty} r^{-1}\,g(r) = 0$. We observe that
\begin{align*}
g(r) &= \int_{S^d} \sigma\big(r\,w^T\nu + b\big)\,\mu(\d w\otimes \d b)
\end{align*}
is a linear combination of terms of the form $\sigma(r\,(w^T\nu)+b)$ whose one-dimensional Barron norm is bounded by $|w^T\nu| + |b| \leq |w|+|b| =1$ if we normalize $\R^d$ with respect to the $\ell^1$-norm (and similarly for other norms). By taking the linear combination, we obtain
\[
\|g\|_{\B(\R)} \leq \|f\|_{\B(\R^d)}.
\]
If a different norm is chosen on $\R^d$, a dimension-dependent factor occurs.

{\bf One-dimensional case.} Recall by Remark \ref{remark equivalent norm 1d} that $g$ has a second derivative in the space of Radon measures and that
\[
\|g\| = |g(0)| + |g'(0)| + \int_\R \big(1+ |\xi|\big)\,\ |g''|(\d\xi)
\]
is an equivalent norm on Barron space in one dimension. We note that 
\begin{align*}
h(x) &= h(0) + h'(0)\, x + \int_0^x h'(\xi)\,\d\xi\\
	&= h(0) + h'(0)\,x + x\,h'(x) - \int_0^x \xi\,h''(\xi)\,\d\xi\\
	&= \big[h'(x) + h'(0)\big]x + \left[ h(0) -\int_0^x \xi\,h''(\xi)\,\d\xi\right]
\end{align*}
for smooth functions $f:\R\to\R$. By approximation, the identity carries over to Barron space, if both $f'(x)$ and the integral $\int_0^x\,|g''|(\d\xi)$ are interpreted either as left or right continuous functions. In particular, for Barron functions we find that 
\begin{equation}
\big| h(x) - \big[h'(x) + h'(0)\big]x\big| \leq \|h\|_\B.
\end{equation}
Additionally, we observe that
\begin{align*}
\big|h'(x) - h'(y)\big| &=\left| \int_x^y h''(s) \ds \right|\\
	&\leq \int_x^y s\,|h''(s)|\,\frac1s\ds\\
	&\leq \frac{\|h\|_\B}{\min\{x,y\}} 
\end{align*}
for $x, y>1$. We distinguish two cases:
\begin{enumerate}
\item There exists $x>1$ such that 
\[
\big|g'(x) + g'(0)\big| > \frac{\|g\|_\B}{x}.
\]
Then
\[
\big|g'(0) + g'(y)\big| \geq \big|g'(0) + g'(x)\big| - \big|g'(x) - g'(y)\big| \geq \big|g'(x) + g'(0)\big| - \frac{\|g\|_\B}{x} =: \eps>0
\]
for all $y\geq x$. In particular, 
\[
\liminf_{x\to\infty} \left|\frac{g(x)}{x}\right| \geq \eps,
\]
contradicting the assumption that $g$ grows sublinearly.

\item For all $x>1$, the estimate 
\[
\big|g'(x) + g'(0)\big| \leq \frac{\|g\|}{x}
\]
holds. Then also
\[
|g(x)| \leq \big|g'(x) - g'(0)\big|\,|x| + |g(0)| + \int_0^\infty |\xi|\,|g''|(\d \xi) \leq 2 \,\|g\|
\]
holds, i.e.\ $g$ is bounded.
\end{enumerate}
\end{proof}

\begin{corollary}
Let $f\in \B(\R^d)$. Then $f$ is a sum of a bounded and a positively one-homogeneous function:
\[
f = f_\infty + \big[f-f_\infty\big] \qquad \|f-f_\infty\|_{L^\infty(\R^d)} \leq 2\,\|f\|_\B.
\]
\end{corollary}

\subsection{Barron functions which decay at infinity}

It is well-known that there are no finite two-layer neural networks with compact support in $\R^d$ for any $d\geq 2$ -- see e.g.\ \cite{lu2021note} for an elementary proof of a (much) stronger result. The same proof applies e.g.\ to finite sums of trigonometric functions $\sin(\xi\cdot x)$ and $\cos (\xi\cdot x)$, which cannot be compactly supported in $\R^d$ even for $d=1$. Fourier transforms on the other hand, the natural analogue to Barron functions, can represent any function in Schwartz space, in particular any compactly supported and infinitely smooth function. It is therefore not obvious whether compactly supported Barron functions exist. We answer a weaker question as follows.

\begin{lemma}\label{lemma barron decay 1}
For any $d\geq 1$, the function $f:\R^d\to \R$, $f(x) = \big(|x|^2+1\big)^{-1/2}$ is in Barron space and $\|f\|_{\B(\R^d)}\leq C\sqrt{d}$ for a constant $C>0$ which does not depend on $d$. 
\end{lemma}

\begin{proof}
We note that the function $h:\R\to\R$, $h(z) = \exp(-z^2/2)$ is in $\B(\R)$ since
\[
|h(0)| + \int_{-\infty}^\infty|z|\,|h''(z)|\,\dz <\infty.
\]
For every fixed $\nu\in\R^d$, the function
\[
f_\nu:\R^d \to \R, \qquad f_\nu(x) = h(\nu\cdot z)
\]
satisfies $\|f_\nu\|_\B \leq |\nu|\,\|h\|_\B$, so the Gaussian average $f(x)$ of the values $f_\nu(x)$ defines a function $f$ such that 
\[
f(x) =  \frac1{(2\pi)^{d/2}} \int_{\R^d} h(z\cdot \nu_1 e_1)\,\exp\left(-\frac{|\nu|^2}2\right)\d\nu.
\]
By H\"older's inequality for the Gaussian measure, the Barron bound 
\begin{align*}
\|f\|_\B &\leq \frac1{(2\pi)^{d/2}} \int_{\R^d}\|h\|_\B|\nu|\exp\left(-\frac{|\nu|^2}2\right)\d\nu
	\leq \left(\frac1{(2\pi)^{d/2}} \int_{\R^d}|\nu|^2\,\exp\left(-\frac{|\nu|^2}2\right)\d\nu\right)^\frac12 \|h\|_\B.
\end{align*}
holds. Since
\[
\frac1{\sqrt{2\pi}} \int_\R z^2\,\exp\left(-\frac{z^2}2\right)\dz = \frac1{\sqrt{2\pi}} \int_\R \exp\left(-\frac{z^2}2\right)\dz = 1,
\]
the bound $\|f\|_{\B(\R^d)} \leq \sqrt{d}\,\|h\|_{\B(\R)}$ follows by Fubini's theorem.
 We can explicitly compute $f$ as
\begin{align*}
f(x) &= f(|x|\,e_1) \\
	&= \frac1{(2\pi)^{1/2}} \int_{\R}h(z\cdot \nu_1 e_1)\,\exp\left(-\frac{\nu_1^2}2\right)\d\nu_1\\
	&= \frac1{(2\pi)^{1/2}} \int_\R \exp\left( - \frac{\big(|x|^2+1\big)\nu_1^2}2\right)\d\nu_1\\
	&= \frac1{\sqrt{|x|^2+1}}\frac1{(2\pi)^{1/2}} \int_\R\exp\left( - \frac{\big(\sqrt{|x|^2+1}\,\nu_1\big)^2}2\right)\sqrt{|x|^2+1}\,\d\nu_1\\
	&= \frac1{\sqrt{|x|^2+1}}\frac1{(2\pi)^{1/2}} \int_\R\exp\left(-\frac{z^2}2\right)\dz\\
	&= \frac1{\sqrt{|x|^2+1}}.
\end{align*}
\end{proof}

If we consider 
\[
f(x) = \frac1{\sqrt{|x-z_1|^2+1}} - \frac1{\sqrt{|x-z_2|^2+1}},
\]
then $f$ is a Barron function which decays like $O(|x|^{-2})$ at infinity. The result admits a refinement which was pointed out to us by Jonathan Siegel.

\begin{lemma}\label{lemma barron decay 2}
For any $d\geq 1$, and $k\geq 0$ there exists a non-zero Barron function $f:\R^d\to\R$ and a constant $C_k>0$ such that
\[
\sup_{|x|\geq r} |f(x)| \leq \frac{ C\,|f(0)|}{r^{2k+1}} \qquad \forall \ r\geq 2^{-1/2}.
\]
\end{lemma}

\begin{proof}
{\bf Step 1.}
Assume for the moment that a Barron function $h:\R\to\R$ with the following properties exists:
\begin{enumerate}
\item $h$ is supported in $[-1,1]$,
\item $h(0)\neq0$, and
\item $\int_{-1}^1 h(y)\,y^{2j} \dy = 0$ for $j=0, \dots, k-1$.
\end{enumerate}
As before, we define 
\[
f(x) = \frac1{(2\pi)^{d/2}}\int_{\R^d} h(x\cdot \nu)\,\exp\left(-\frac{|\nu|^2}2\right)\,\d\nu
\]
and find that
\begin{align*}
f(r e_1) &=\frac1{\sqrt{2\pi}} \int_{-\infty}^\infty h(r\,\nu_1)\,\exp\left(-\frac{|\nu|^2}2\right)\,\d\nu\\
	&= \frac1{\sqrt{2\pi}\,r} \int_{-\infty}^\infty h(y)\,\exp\left(-\frac{y^2}{2r^2}\right)\dy\\
	&= \frac1{\sqrt{2\pi}\,r} \sum_{n=0}^\infty \frac{(-1)^n}{n!\,(2r^2)^n} \int_{-1}^1 h(y)\,y^{2n}\dy\\
	&= \frac1{\sqrt{2\pi}\,r} \sum_{n=k}^\infty \frac{(-1)^n}{n!\,(2r^2)^n} \int_{-1}^1 h(y)\,y^{2n}\dy\\
	&=  \frac1{\sqrt{2\pi}\,2^{2k}\,r^{2k+1}} \sum_{n=k}^\infty \frac{(-1)^{n}} {n!\,(2r^2)^{n-k}}\int_0^1 h(y)\,y^{2k}\,y^{2(n-k)}\dy\\
	&\leq \frac1{\sqrt{2\pi}\,2^{2k}\,r^{2k+1}} \sum_{n=k}^\infty \frac{1}{n!\,(2r^2)^{n-k}} \|h\|_{L^\infty}\\
	&\leq \frac{\|h\|_{L^\infty}}{\sqrt{2\pi}\,2^{2k}\,r^{2k+1}}\left(\sum_{n=k}^\infty \frac1{n!}\right).
\end{align*}
for all $r\geq 2^{-1/2}$. Thus the Lemma is proved, assuming that a suitable Barron function in one dimension can be found.

{\bf Step 2.} In this step, we show that a suitable $h$ exists. The linear map
\[
A_k :\B(\R) \to \R^{k+2}, \qquad A_k(h) = \begin{pmatrix} h(1)\\ h(-1)\\ \int_{-1}^1 x^2\,h(x)\dx\\ \vdots\\ \int_{-1}^1 h(x) \,x^{2(k-1)}\dx\end{pmatrix}
\]
has a non-trivial kernel $V_k \subseteq \B(\R)$. Since $h(\pm 1) = 0$, we can modify $h$ by adding multiples of $\sigma\big(\pm(x-1)\big)$ to ensure that $h$ is supported in $[-1,1]$. Any $h\in V_k$ induces $f$ such that $\lim_{r\to\infty} r^{2k+1}\,f(re_1) =0$. The only question is whether there exists $h$ such that $f\not \equiv 0$. This is true for example if $f(0) = h(0) \neq 0$.

Assume for the sake of contradiction that $A_k(h) =0$ implies that $h(0)=0$. This is the case if and only if there exist coefficients $a_0,\dots, a_{k-1}, b_1, b_{-1}$ such that
\[
h(0) = \sum_{j=0}^{k-1} a_j \int_{-1}^1 y ^{2j} h(y)\dy + b_1\,h(1) + b_{-1}\,h(-1)
\]
for all $h\in \B(\R)$. We can show that this is not the case by considering
\[
h_\eps(y) = \max\left\{ 1 - \frac {|x|}\eps, 0\right\} = \frac1\eps\,\sigma\left(x+\eps\right) - \frac2\eps \,\sigma(x) + \frac1\eps\,\sigma(x-\eps).
\]
\end{proof}

 While the Barron functions we construct may not be compactly supported, they exhibit properties which are very different from those of finite two-layer networks. In particular, we note the following.

\begin{corollary}
For any $d\geq 2$, there exists $f\in \B(\R^d)$ such that $f\not\equiv 0$ and $f\in L^1(\R^d) \cap L^\infty(\R^d)$.
\end{corollary}

These functions can be approximated efficiently by finite two-layer networks in $L^2(\P)$ if the data distribution $\P$ has finite second moments by Theorem \ref{direct approximation theorem}. Whether there exist Barron functions which are compactly supported, or non-negative Barron functions which decay faster than $|x|^{-1}$ at infinity, remains an open problem.

\subsection{Singular set}

We show that the singular set of a Barron function (the set where the function is not differentiable) is fairly small and easy to understand. Again, we write $(w,b)$ explicitly instead of $w$ and $(x, 1)$ instead of $x$ to understand the finer properties of Barron functions.

\begin{lemma}
Let $f$ be a Barron function on a domain $\Omega\subseteq \R^d$. Then for every $x\in \Omega$ and every $v\in \R^d$, the one-sided derivatives $\partial^\pm_vf(x)$ exists and
\begin{align*}
\partial^+_vf(x) &:= \lim_{h\searrow 0} \frac{f(x+hv)- f(x)}h\\
	&= \int_{A_x^+}\langle w,v\rangle \,\mu(\d w\otimes \d b) + \int_{A_x^0} \sigma(\langle w,v\rangle) \,\mu(\d w\otimes \d b)
\end{align*}
where $f= f_\mu$ and
\[
A_x^+ := \{(w,b)\:|\:\langle w,x\rangle + b >0\}, \qquad A_x^0 := \{(w,b)\:|\:\langle w,x\rangle + b =0\}.
\]
The jump of the derivatives is
\[
[\partial_vf]_x := \partial_v^+f(x) - \partial_v^-f(x) = \int_{A_x^0} |\langle w,v\rangle| \,\mu(\d w\otimes \d b).
\]
\end{lemma}

\begin{proof}
We observe that
\[
\lim_{t\to 0^+} \frac{\sigma(a+tb) - \sigma(a)}{t} = \begin{cases} \sigma(b) &a=0\\ b &a>0\\ 0 &a<0\end{cases}, \qquad \lim_{t\to 0^-} \frac{\sigma(a+tb) - \sigma(a)}{t} = \begin{cases} -\sigma(-b) &a=0\\ b &a>0\\ 0 &a<0\end{cases}.
\]
Note that 
\begin{align*}
\partial^+_vf(x) &:= \lim_{h\searrow 0} \frac{f(x+hv)- f(x)}h\\
	&= \lim_{h\searrow 0} \int \frac{\sigma(\langle x,w\rangle + b + h\langle v,w\rangle) - \sigma(\langle x,w\rangle+b)}h\,\mu(\d w\otimes \d b).
\end{align*}
Since both $\mu^+, \mu^-$ are finite measures, we may use the dominated convergence theorem with majorizing function $|v|$ to take the limit inside. This proves the first part of the theorem. The second part follows immediately noting that 
\begin{align*}
[\partial_vf]_x &= \partial_v^+f(x) - \partial_v^-f(x)\\
	&= \int_{\{(w,b)\:|\:\langle w,x\rangle + b >0\}}\langle w,v\rangle \,\mu(\d w\otimes \d b) + \int_{\{(w,b)\:|\:\langle w,x\rangle + b =0\}} \sigma(\langle w,v\rangle) \,\mu(\d w\otimes \d b)\\
	&\qquad - \int_{\{(w,b)\:|\:\langle w,x\rangle + b >0\}}\langle w,v\rangle \,\mu(\d w\otimes \d b) - \int_{\{(w,b)\:|\:\langle w,x\rangle + b =0\}} -\sigma(-\langle w,v\rangle) \,\mu(\d w\otimes \d b)\\
	&= \int_{A_x^0} \sigma(\langle v,w\rangle) + \sigma(-\langle v,w\rangle)\,\mu(\d w\otimes \d b).
\end{align*}
\end{proof}

\begin{corollary}
Let $\mu$ be a finite signed measure on $S^d$ such that $\mu(S^d\cap H) =0$ for every hyperplane $H$ in $\R^{d+1}$. Then $f_\mu$ is $C^1$-smooth on $\R^d$.
\end{corollary}

\begin{proof}
The function
\[
(x,v)\mapsto (\partial_vf)(x) = \int_{S^d}\sigma(w^Tx)\,\mu(\d w)
\]
is continuous by the dominated convergence theorem.
\end{proof}

Philosophically, it makes sense that only the singularity in $\sigma$ contributes to the singularity of $f_\mu$, and not the segments where $\sigma$ is linear. A single neuron activation $\sigma(w^Tx +b)$ is differentiable except along the hyperplane $\{x: w^Tx+b=0\}$. Similarly, finite two-layer networks have a singular part which is contained in a union of hyperplanes. We will show that a similar result holds for general Barron functions.

\begin{theorem}
Let $\mu$ be a Radon measure on $S^d$. We can decompose $\mu = \sum_{i=0}^\infty \mu_i$ and $f_\mu = \sum_{i=0}^\infty f_{\mu_i}$ in such a way that
\sw{ 
\begin{itemize}
\item $f_{\mu_0}$ is $C^1$-smooth,
\item for $i\geq1$, $\mu_i$ is supported on the intersection of $S^d$ and a $k_i$-dimensional affine subspace $w_i+ V_i$ of $\R^{d+1}$ for some $0\leq k_i\leq d-1$,
\item $f_{\mu_i}$ is smooth on $V_i\cap\R^d$ except at a single point $x_i$ and constant in directions $w\in V_i^\bot\cap \R^d$, so
\item the singular set $\Sigma_i$ of $f_{\mu_i}$ is the $d-k_i$-dimensional affine subspace $x_i + V_i^\bot$  of $\R^d$.
\end{itemize}
}
\end{theorem}

\begin{proof}
{\bf Step 0.} For Barron functions of one real variable, \sw{$S^d = S^1$ is the circle. Since a finite measure has only countably many atoms (which can be represented as intersections of $S^1$ with a one-dimensional affine subspace of $\R^2$), the representation holds with $\mu_i \ll \delta_{w_i}$ for the atoms $w_i\in S^1$ of $\mu$ and $\mu_0 = \mu - \sum_i \mu_i$.}

We proceed by induction. Assume that the Theorem is proved for $k\leq d-1$.

{\bf Step 1.} First, we decompose the measure $\mu$ into lower-dimensional strata. Since the total variation measure $|\mu|$ is finite, there are only finitely many atoms of a certain size $\eps>0$ of $|\mu|$, i.e.\ only finitely many points $x_1,\dots, x_N\in S^d$  such that $|\mu|(\{x_i\}) \geq \eps$. As a consequence, the set $A = \{x_1,x_2, \dots\}$ of atoms of $|\mu|$ is at most countable. We define
\[
\mu_{0,i} := \mu|_{\{x_i\}},\qquad \tilde\mu_1 = \mu - \sum_{i=1}^\infty\mu_{0,i}.
\]
In particular, $\tilde\mu_1$ does not have any atoms and
\[
\|\mu\| = \|\tilde\mu_1\| + \sum_{i=1}^\infty \|\mu_{0,i}\|
\]
since the measures are mutually singular. Now we claim that there exist at most countably many circles $s^1_1, s^1_2, \dots$ in $S^{d}$ such that $|\mu| (s^1_i)>0$, where a circle is the intersection of $S^d$ with a two-dimensional affine space. If there were uncountably many circles of positive measure, there would be $\eps>0$ such that uncountably many circles have measure $\geq \eps$, just like for atoms. Since circles are either disjoint or intersect in one or two points, they intersect in $|\tilde \mu^1|$-null sets. So if there were infinitely many circles $s_1, s_2,\dots$ such that $|\tilde\mu^1|(s_i)\geq \eps$ for all $i\in\N$, then
\[
\|\mu\| \geq \|\tilde\mu^1\| \geq \sum_{i=1}^\infty |\tilde\mu^1| (s_i) = \infty,
\]
leading to a contradiction. We now define
\[
\mu_{1,i} = \mu|_{s_i}, \qquad \tilde \mu_2 = \tilde\mu_1 - \sum_{i=1}^\infty \mu_{1,i}.
\]
We iterate this procedure, using that spheres of dimension $k$ intersect in spheres of dimension $\leq k-1$ to obtain a decomposition
\[
\mu = \tilde \mu_d + \sum_{k=0}^{d-1} \sum_{i=1}^\infty \mu_{k,i}
\]
where the inner sum may be finite or countable and for all $i$. If it is finite, we set $\mu_{k,i}$ to be the zero measure on a subspace of the correct dimension and ignore the distinction notationwise.

{\bf Step 2.} Fix indices $k,i$ and the affine space $W_{k,i}$ of dimension $k$ such that $\spt(\mu_{k,i}) = S^d\cap W_{k,i}$. By construction, the function 
\[
f_{\mu_{k,i}}(x) = \int_{W^k_i \cap S^d} \sigma(w^Tx+ b) \,\mu(\d w\otimes \d b)
\]
is constant in directions orthogonal to $W^k_i$, i.e.\ $f_{\mu^k_i}(x+v) = f_{\mu^k_i}(x)$ if $v$ is orthogonal to the projection $\widehat W^k_i$ of $W^k_i$ onto the $\{b=0\}$-plane. $\widehat W^k_i$ has dimension $k$, unless the `bias direction' $(0,\dots,0,1)$ is in $W^k_i$, in which case it has dimension $k-1$. 

In either case, $f_{\mu_{k,i}}$ is a Barron function of $\leq d-1$ variables. By the induction hypothesis, we can write 
\[
f_{\mu_{k,i}} = \sum_{j=0}^\infty f_{\mu^k_{i,j}}
\]
where the singular set of $f_{\mu^k_{i,j}}$ is an affine subspace of $W^k_i$ of dimension $\leq d-1$. Thus the theorem is proved.
\end{proof}

\begin{remark}
The singular set $\Sigma_f$ is contained in the union $\bigcup_i \Sigma_{f_{\mu_i}}$ which may be empty or not. In particular, the Hausdorff dimension of the singular set of $f$ is an integer $k\in \{0,\dots, d-1\}$.
\end{remark}

\begin{remark}
We need to consider the decomposition of the singular set since there may be cancellations between the singularities of different dimensionality. For example, the singular set of the Barron function $f(x) = |x_1| - \sqrt{x_1^2+ x_2^2}$ is $\Sigma = \{x_1 = 0\}\setminus\{(0,0)\}$ and not a union of affine spaces.
\end{remark}

\begin{remark}
The singular set of a single neuron activation has dimension $d-1$. In Example \ref{example homogeneous} we present examples of Barron functions whose singular set is a linear space of strictly lower dimension.
\end{remark}

\begin{remark}
$\Sigma$ may be dense in $\Omega$. For example, the primitive function of any bounded monotone increasing function on $[0,1]$ with a dense set of jump discontinuities is in $\B[0,1]$ and has a dense singular set.
\end{remark}

\begin{remark}
Barron functions cannot have curved singular sets of co-dimension $1$. In particular, functions like
\[
f_1(x) = \dist(x, S^{d-1}), \qquad f_2(x) = \dist(x, B_1(0))
\]
are not Barron functions, where the distance function, sphere and unit ball are all with respect to the Euclidean norm.
\end{remark}

\begin{remark}
For $d\geq 3$, the function $f(x) = \max\{x_1,\dots, x_d\}$ is not a Barron function over $[-1,1]^d$. Namely, the singular set 
\[
\Sigma = \bigcup_{i\neq j} \{x : x_i = x_j = f(x)\}
\]
is incompatible with the linear space structure if there exists a third dimension.
Note, however, that $f$ can be represented by a network with $\lceil \log_2(d)\rceil$ hidden layers since
\[
\max\{x_1, x_2\} = x_1 + \sigma(x_2-x_1), \qquad \max\{x_1, x_2, x_3, x_4\} = \max\big\{\max\{x_1, x_2\}, \max\{x_3,x_4\}\big\}
\]
and so on.
\end{remark}

\begin{remark}\label{remark not an algebra}
Note that $\sigma(x_1)$ and $\sigma(x_2)$ are Barron functions, but the singular set of their product is the corner
\[
\Sigma = \{x_1= 0, x_2\geq 0\} \cup \{x_1\geq 0, x_2 =0\}.
\]
Thus $\sigma(x_1)\,\sigma(x_2)$ is not a Barron function on $[-1,1]^2$. In particular, Barron space in dimension $d\geq 2$ is generally not an algebra.
\end{remark}

Barron-type spaces for deep neural networks will be developed in detail in a forth-coming article \cite{deep_barron}. We briefly discuss three-layer Barron networks using examples of functions which need to be in any reasonable space of infinitely wide three-layer network.

\begin{remark}
Three-layer networks have much more flexible singular sets. For $\lambda>0$, the function
\[
f(x) = \min\{\sigma(x_1), \: \sigma(x_2 - \lambda x_1)\}
\]
has a singular set given by the union of three half-lines 
\[
\Sigma = \{0< x_1= x_2 - \lambda x_1\} \cup \{x_1 =0, x_2-\lambda x_1>0\}\cup \{x_2 - \lambda x_1 =0, x_1 \geq 0\}
\]
since the function is zero everywhere outside of the quadrant $\{x_1, x_2>0\}$. Since $x\mapsto x^2$ is a Barron function on bounded intervals,
\[
f(x,y) = \sigma (y-x^2)
\]
is a three-layer network \sw{on any compact subset of $\R^2$ with one infinite and one finite layer -- see \cite{deep_barron} for the precise concept of infinitely wide three-layer networks and \cite[Lemma 3.12]{deep_barron} for a statement on the composition of Barron functions.} Here the singular set is the curve $\{y= x^2\}$. Stranger examples like 
\[
f(x,y) = \sigma(y-x^2) + \sigma(x-y^2)
\]
are also possible. Generally, the singular sets of three-layer networks are at least as flexible as the level sets of two-layer networks since for a Barron function $f$ and a real value $y$, the function
\[
x \mapsto \big| f(x) - y\big|
\]
is a three-layer network whose singular set is the union of the singular set of $f$ and the level set $\{f=y\}$.
\end{remark}

\subsection{Applications of the structure theorem}

The structure theorem for non-differentiable Barron functions allows us to characterize the class of morphisms which preserve Barron space.

\begin{theorem}
Let $\psi:\R^d\to \R^d$ be a $C^1$-diffeomorphism and 
\[
A_\psi:C^{0,1}(\R^d) \to C^{0,1}(\R^d), \qquad A_\psi(f) = f\circ\psi.
\]
Then $A_\psi(\B) \subseteq \B$ if and only if $\psi$ is affine. In particular, $A_\psi(\B) = \B$.
\end{theorem}

\begin{proof}
If $A_\psi(\B)\subseteq \B$, then $f(x):= \sigma\big(w^T\psi(x) + b\big)$ is a Barron function for any $w\in\R^d$ and $b\in \B$. As $\psi$ is a diffeomorphism, the singular set of $f$ coincides with the level set 
\[
A_{(w,b)} := \big\{x : w^T\psi(x) + b = 0\big\}.
\]
We conclude that for any $w, b\in \R^d$, the level set $A_{(w,b)}$ is a hyperplane in $\R^d$. In particular, for $1\leq i\leq d$, we find that the level sets of $\psi_i (x) = e_i \cdot \psi(x)$ are parallel hyperplanes. This means that there exist vectors $v_i$ for $1\leq i\leq d$ and functions $\phi_d:\R\to\R$ such that $\psi_i(x) = \phi_i\big(v_i^Tx\big)$. 

We note that also $(e_1 + e_2)\cdot \psi(x) = \phi_1(v_1^Tx) + \phi_2(v_2^Tx)$ has level sets which are hyperplanes, i.e.
\[
\phi_1(v_1^Tx) + \phi_2(v_2^Tx) = \tilde \phi(\tilde v^Tx).
\]
By regularity, all functions are $C^1$-smooth, and since $\psi$ is a diffeomorphism, they are strictly monotone. The diffeomorphism property also implies that $v_1, v_2$ are linearly independent. In view of these properties, $\tilde v$ cannot be a multiple of $v_1$ or $v_2$. We compute
\[
\tilde\phi'(\tilde v^Tx)\,\tilde v = \phi_1'(v_1^Tx)\,v_1 + \phi_2'(v_2^Tx)\,v_2.
\]
Choose $w$ in the plane spanned by $v_1, v_2$ such that $w$ is orthogonal to $\tilde v$. Then
\[
\tilde\phi'(\tilde v^Tx)\,\tilde v = \tilde\phi'(\tilde v^T(x+\lambda w))\,\tilde v = \phi_1'(v_1^Tx+ \lambda v_1^Tw)\,v_1 + \phi_2'(v_2^Tx+\lambda v_2^Tw)\,v_2.
\]
Assuming that $\phi_1, \phi_2$ are $C^2$-smooth, we differentiate the identity with respect to $\lambda$ we obtain that
\begin{equation}\label{eq linear identity}
0 = (v_1^Tw)\, \phi_1''(v_1^Tx+ \lambda v_1^Tw)\,v_1 + (v_2^Tw)\,\phi_2'(v_2^Tx+\lambda v_2^Tw)\,v_2 \qquad\forall\ \lambda\in\R.
\end{equation}
Since $v_1, v_2$ are linearly independent, this can only be satisfied if $\phi_1'' = \phi_2'' = 0$, i.e.\ if and only if $\phi_1$ and $\phi_2$ are both linear. If $\phi_1$ or $\phi_2$ is not $C^2$-smooth, we can mollify \eqref{eq linear identity} as a function of $\lambda$ before we differentiate. The constant vectors $v_1, v_2$ are not affected by the mollification, so we conclude that any mollification of $\phi_1$ and $\phi_2$ must be linear. As before, we conclude that $\phi_1$ and $\phi_2$ are linear functions.

Since $\psi(x) = \big(\phi_i(v_i^Tx)\big)_{i=1}^d$ and all coefficient functions $\phi_i$ are linear, the whole map $\psi$ is linear.
\end{proof}

The structure theorem can further be used to show that the Barron property cannot be `localized' in the same way as classical smoothness criteria. This observation was made previously in \cite{wojtowytsch2020some}.

\begin{example}
Let $U\subseteq \R^2$ be a U-shaped domain, e.g.\ 
\[
U = \R^2 \setminus \{x_1=0, x_2\geq 0\}.
\]
Then the function
\[
f: U \to \R, \qquad f(x) = \begin{cases} \sigma (x_2) & x_1>0\\ 0 &x_1\leq 0\end{cases}
\]
has the following property: {\em Every $x\in U$ has a neighbourhood $V$ such that $f$ is a Barron function on $V$.} However, $f$ is not a Barron function on $U$ since the singular set of $f$ would need to contain the intersection of $U$ with the line $\{x : x_2=0\}$. This example can be generalized to other domains where it may be less obvious.
\end{example}

\section*{Acknowledgements}

This work is supported in part by a gift to Princeton University from iFlytek. \sw{SW would like to thank Jonathan Siegel for helpful discussions.}


\bibliographystyle{../../alphaabbr}
\bibliography{../../NN_bibliography}

\end{document}